\documentclass{article}

\usepackage{iclr2027_conference,times}
\iclrfinalcopy

\usepackage[sets,nn,operations,colors,tikz]{cora-macs}
\usetikzlibrary{positioning}  % for \node[below left=... and ... of ...]
\usetikzlibrary{shapes.callouts}  % for speech-bubble callouts
\usetikzlibrary{decorations.pathmorphing}  % for wavy / snake arrows
\usetikzlibrary{calc}  % for ($A+(dx,dy)$) coordinate arithmetic
\CORAexternalizeFiguresByDefault
\usepackage{hyperref}

\usepackage{xcolor}
\usepackage{lipsum}
\setlipsum{%
    par-before = \begingroup\color{lightgray},
    par-after = \endgroup,
    sentence-before = \begingroup\color{lightgray},
    sentence-after = \endgroup
}

\newcommand{\todoCiteVerify}[2][]{{\color{orange}\ifx\relax#1\relax\citep{#2}\else\citep[#1]{#2}\fi?}}

\usepackage{pifont}

\usepackage{amsmath}
\usepackage{mathtools}  % \coloneqq and friends (loads amsmath)
\DeclareMathOperator*{\argmin}{arg\,min}

\usepackage{amsthm}
\usepackage{cleveref}
\usepackage{amsfonts}
\usepackage{wasysym}

\newtheorem{proposition}{Proposition}

\theoremstyle{definition}
\newtheorem{definition}{Definition}
\theoremstyle{remark}

\crefname{section}{Sec.}{Sec.}
\crefname{subsection}{Sec.}{Sec.}
\crefname{figure}{Fig.}{Fig.}
\crefname{algorithm}{Alg.}{Alg.}
\crefname{table}{Tab.}{Tab.}
\crefname{example}{Ex.}{Ex.}
\crefname{equation}{Eq.}{Eq.}
\crefformat{equation}{Eq.~#2#1#3}
\Crefformat{equation}{Eq.~#2#1#3}
\crefrangeformat{equation}{Eqs.~#3#1#4 to~#5#2#6}
\Crefrangeformat{equation}{Eqs.~#3#1#4 to~#5#2#6}
\crefmultiformat{equation}{Eqs.~#2#1#3}{ and~#2#1#3}{, #2#1#3}{, and~#2#1#3}
\Crefmultiformat{equation}{Eqs.~#2#1#3}{ and~#2#1#3}{, #2#1#3}{, and~#2#1#3}
\crefname{definition}{Def.}{Def.}
\crefname{proposition}{Prop.}{Prop.}
\crefname{corollary}{Cor.}{Cor.}
\crefname{theorem}{Thm.}{Thm.}
\crefname{lemma}{Lemma}{Lemmas}
\crefname{appendix}{Appendix}{Appendix}

\renewcommand{\Cref}[1]{\cref{#1}}

\usepackage{booktabs} % requires 'booktabs' package
\usepackage{amstext} % for \text macro in headers
\usepackage{array} % for \newcolumntype macro
\usepackage{multirow} % for \multirow in ablation tables

\newcolumntype{L}{>{$}l<{$}} % math-mode version of "l" column type
\newcolumntype{R}{>{$}r<{$}} % math-mode version of "r" column type
\newcolumntype{C}{>{$}c<{$}} % math-mode version of "c" column type

\usepackage{xfrac}

\usepackage{algorithm}
\usepackage[noend]{algpseudocode}
\usepackage[commentColor=black, italicComments=false]{algpseudocodex}
\AtBeginEnvironment{algorithm}{\tikzexternaldisable}
\AtEndEnvironment{algorithm}{\tikzexternalenable}

\usepackage{thm-restate}

\newcommand{\token}[1]{\texttt{\small #1}}

\newcommand{\mathOp}[1]{\ensuremath{\mathrm{#1}}}
\newcommand{\atLayer}[2][\layerIdx]{\ensuremath{#2_{#1}}}
\newcommand{\sctag}[1]{\textsc{#1}}

\newcommand{\IRN}{\text{IRN}}                       % interpretable replacement network
\newcommand{\TC}{\text{TC}}                         % transcoder
\newcommand{\SAE}{\text{SAE}}                       % sparse autoencoder
\newcommand{\MLP}{\text{MLP}}                       % MLP block
\newcommand{\Attn}{\text{Attn}}                     % attention block

\newcommand{\numTokens}{\tau}                           % number of tokens
\newcommand{\numModelDim}{d_\textnormal{model}}         % residual-stream dimension (d_model)
\newcommand{\numMLPdim}{\ensuremath{d_{\MLP}}}          % MLP hidden width (post-gating where applicable)
\newcommand{\numIRNdim}{\ensuremath{d_{\IRN}}}          % IRN dictionary size (formerly m / d_feat)
\newcommand{\numLayersMLP}{\ensuremath{\numLayers_{\MLP}}} % number of MLP layers (kappa_MLP)
\newcommand{\layerIdx}{k}                               % current layer index
\newcommand{\layerMLP}[1][\layerIdx]{\atLayer[#1]{\MLP}}
\newcommand{\layerIRN}[1][\layerIdx]{\atLayer[#1]{\IRN}}
\newcommand{\layerTC}[1][\layerIdx]{\atLayer[#1]{\TC}}
\newcommand{\layerSAE}[1][\layerIdx]{\atLayer[#1]{\SAE}}
\newcommand{\hAt}[1][\layerIdx]{\atLayer[#1]{H}}              % residual stream (d x #token) after layer k
\newcommand{\hMid}[1][\layerIdx]{\atLayer[#1]{H'}}            % post-attention, pre-MLP intermediate
\newcommand{\hAdv}[1][\layerIdx]{\atLayer[#1]{\widetilde{H}}} % adversarial / IRN-approx

\newcommand{\nnOutBound}{\delta}
\newcommand{\jacc}{\ensuremath{J}}                      % top-K Jaccard
\newcommand{\jaccTrue}{\ensuremath{\jacc^*}}    % attack-derived upper bound on worst-case J
\newcommand{\jaccUp}{\ensuremath{\overline{\jacc}}}    % attack-derived upper bound on worst-case J
\newcommand{\jaccLo}{\ensuremath{\underline{\jacc}}}   % certified lower bound on worst-case J
\newcommand{\numTop}{\ensuremath{K}}                    % number of top features (= 20); K, distinct from layer index \layerIdx=k
\newcommand{\topk}{\ensuremath{\text{top-}\numTop}}

\newcommand{\pertMinor}{\sctag{minor}}
\newcommand{\pertMedium}{\sctag{medium}}
\newcommand{\pertMajor}{\sctag{major}}

\newcommand{\deltaTrue}{\delta^*}
\newcommand{\deltaVerified}{\overline{\delta}}
\newcommand{\deltaPGD}{\underline{\delta}}
\newcommand{\nFree}{K'}
\newcommand{\Tact}{T_{\mathrm{act}}}
\newcommand{\Tina}{T_{\mathrm{inact}}}
\newcommand{\Tund}{T_{\mathrm{und}}}
\newcommand{\trainStd}{\sctag{Std}}
\newcommand{\trainPGD}{\sctag{PGD}}                       % adversarial / Madry-style
\newcommand{\trainSet}{\sctag{Set}}                       % set-based / verification-aware
\newcommand{\trainLoRA}{\sctag{LoRA}}                     % low-rank update variant (suffix)

\newcommand{\GPT}{\mbox{GPT-2}}
\newcommand{\GemmaTwo}{\mbox{Gemma 2 2B}}
\newcommand{\GemmaThree}{\mbox{Gemma 3 1B}}
\newcommand{\Llama}{\mbox{Llama 3.2 1B}}
\newcommand{\Qwen}{\mbox{R1-Distill-Qwen 1.5B}}
\newcommand{\QwenShort}{\mbox{Qwen 1.5B}}
\newcommand{\GemmaScope}{\mbox{Gemma Scope}}

\newcommand{\ReLU}{\mathOp{ReLU}}
\newcommand{\GELU}{\mathOp{GELU}}
\newcommand{\GeGLU}{\mathOp{GeGLU}}
\newcommand{\SwiGLU}{\mathOp{SwiGLU}}
\newcommand{\JumpReLU}{\mathOp{JumpReLU}}
\newcommand{\TopK}{\mathOp{TopK}}

\makeatletter
\AtBeginDocument{%
    \@ifundefined{iclrfinaltrue}{}{%
        \@ifundefined{tikzexternalrealjob}{}{%
            \edef\@tikzrobot@jn{\jobname}%
            \edef\@tikzrobot@ej{\tikzexternalrealjob}%
            \ifx\@tikzrobot@jn\@tikzrobot@ej\else\iclrfinaltrue\fi%
        }%
    }%
}
\makeatother

\newcommand{\robotClaw}{%
    \draw[outline, fill=\robotcolor!55, line width=\lwd pt]
        (-0.08, 0.00) -- (-0.08, -0.08) -- (-0.02, -0.08) -- (-0.02, -0.04) --
        ( 0.02, -0.04) -- ( 0.02, -0.08) -- ( 0.08, -0.08) -- ( 0.08, 0.00) -- cycle;
}

\tikzset{
    pics/robot/.style = {
        code = {%
            \tikzset{robot/.cd, #1}%
            \edef\robotcolor{\pgfkeysvalueof{/tikz/robot/color}}%
            \edef\robotpose{\pgfkeysvalueof{/tikz/robot/pose}}%
            \edef\robotface{\pgfkeysvalueof{/tikz/robot/face}}%
            \edef\robotbody{\pgfkeysvalueof{/tikz/robot/body}}%
            \edef\robotheadtilt{\pgfkeysvalueof{/tikz/robot/headtilt}}%
            \edef\robotheadshift{\pgfkeysvalueof{/tikz/robot/headshift}}%
            \edef\robotlarm{\pgfkeysvalueof{/tikz/robot/larm}}%
            \edef\robotrarm{\pgfkeysvalueof{/tikz/robot/rarm}}%
            \edef\robotbob{\pgfkeysvalueof{/tikz/robot/bob}}%

            \def\faceKind{\robotpose}%       default: face = pose name
            \def\poseLeftArm{normal}%        normal | up | extended-left | bent-out
            \def\poseRightArm{normal}%       normal | up | extended-right | bent-up | bent-out

            \ifthenelse{\equal{\robotbody}{auto}}{%
                \edef\effectiveBody{\robotpose}%
            }{%
                \edef\effectiveBody{\robotbody}%
            }%
            \ifthenelse{\equal{\effectiveBody}{thinking}}{%
                \def\poseRightArm{bent-up}%
            }{}%
            \ifthenelse{\equal{\effectiveBody}{pointing-right}}{%
                \def\poseRightArm{extended-right}%
            }{}%
            \ifthenelse{\equal{\effectiveBody}{pointing-left}}{%
                \def\poseLeftArm{extended-left}%
            }{}%
            \ifthenelse{\equal{\effectiveBody}{waving}}{%
                \def\poseRightArm{up}%
            }{}%
            \ifthenelse{\equal{\effectiveBody}{cheering}}{%
                \def\poseLeftArm{up-out}\def\poseRightArm{up-out}%
            }{}%
            \ifthenelse{\equal{\effectiveBody}{shrug}}{%
                \def\poseLeftArm{bent-out}\def\poseRightArm{bent-out}%
            }{}%

            \ifthenelse{\equal{\robotface}{auto}}{%
                \ifthenelse{\equal{\robotpose}{pointing-right}}{\def\faceKind{neutral}}{}%
                \ifthenelse{\equal{\robotpose}{pointing-left}}{\def\faceKind{neutral}}{}%
                \ifthenelse{\equal{\robotpose}{waving}}{\def\faceKind{happy}}{}%
                \ifthenelse{\equal{\robotpose}{cheering}}{\def\faceKind{happy}}{}%
                \ifthenelse{\equal{\robotpose}{shrug}}{\def\faceKind{confused}}{}%
            }{%
                \def\faceKind{\robotface}%
            }%

            \pgfgettransformentries{\rmxx}{\rmxy}{\rmyx}{\rmyy}{\rmtx}{\rmty}%
            \pgfmathsetmacro{\robotscale}{sqrt(abs(\rmxx*\rmyy-\rmxy*\rmyx))}%
            \pgfmathsetmacro{\lwa}{0.65*\robotscale}%
            \pgfmathsetmacro{\lwb}{0.50*\robotscale}%
            \pgfmathsetmacro{\lwc}{0.45*\robotscale}%
            \pgfmathsetmacro{\lwd}{0.40*\robotscale}%
            \pgfmathsetmacro{\lwe}{0.30*\robotscale}%
            \pgfmathsetmacro{\lwf}{0.25*\robotscale}%
            \begin{scope}[
                line cap=butt, line join=miter,
                line width=\lwd pt,   % scaled default for any path without an explicit width
                outline/.style={draw=black, line width=\lwa pt},
                screw/.style  ={draw=black!70, line width=\lwf pt, fill=black!55},
                ledLit/.style ={fill=\robotcolor!50},
            ]
                \fill[black, opacity=0.10] (0, -0.79) ellipse (0.32 and 0.05);
                \fill[black, opacity=0.18] (0, -0.79) ellipse (0.22 and 0.035);

                \begin{scope}[shift={(0, \robotbob)}]

                \begin{scope}[shift={(\robotheadshift, 0)},
                              rotate around={\robotheadtilt:(0, -0.02)}]
                \draw[outline] (-0.22, 0.65) -- (-0.32, 0.86);
                \draw[outline, fill=\robotcolor!75!black]
                    (-0.32, 0.89) circle (0.050);
                \fill[white, opacity=0.6] (-0.335, 0.905) circle (0.014);
                \draw[outline] ( 0.22, 0.65) -- ( 0.32, 0.86);
                \draw[outline, fill=\robotcolor!75!black]
                    ( 0.32, 0.89) circle (0.050);
                \fill[white, opacity=0.6] ( 0.305, 0.905) circle (0.014);

                \draw[outline, fill=\robotcolor!30]
                    (-0.40, 0.65) -- (-0.50, 0.55) -- (-0.50, 0.08) --
                    (-0.40, -0.02) -- ( 0.40, -0.02) -- ( 0.50, 0.08) --
                    ( 0.50, 0.55) -- ( 0.40, 0.65) -- cycle;

                \draw[screw] (-0.41, 0.56) circle (0.022);
                \draw[screw] ( 0.41, 0.56) circle (0.022);
                \draw[screw] (-0.41, 0.05) circle (0.022);
                \draw[screw] ( 0.41, 0.05) circle (0.022);

                \draw[outline, fill=black!85, line width=\lwd pt]
                    (-0.36, 0.55) rectangle (0.36, 0.10);
                \fill[\robotcolor!20] (-0.32, 0.51) rectangle (0.32, 0.14);

               \node[transform shape] at (0, 0.38) {\textbf{AI}};

                \coordinate (-mouth) at (0, 0.20);
                \ifthenelse{\equal{\faceKind}{happy}}{%
                    \fill[ledLit] (-0.10, 0.20) rectangle (-0.04, 0.225);
                    \fill[ledLit] (-0.03, 0.18) rectangle ( 0.03, 0.205);
                    \fill[ledLit] ( 0.04, 0.20) rectangle ( 0.10, 0.225);
                }{}%
                \ifthenelse{\equal{\faceKind}{neutral}}{%
                    \fill[ledLit] (-0.10, 0.19) rectangle ( 0.10, 0.215);
                }{}%
                \ifthenelse{\equal{\faceKind}{confused}}{%
                    \fill[ledLit] (-0.10, 0.18) rectangle (-0.04, 0.205);
                    \fill[ledLit] (-0.03, 0.21) rectangle ( 0.03, 0.235);
                    \fill[ledLit] ( 0.04, 0.18) rectangle ( 0.10, 0.205);
                }{}%
                \ifthenelse{\equal{\faceKind}{surprised}}{%
                    \draw[outline, fill=black, line width=\lwe pt]
                        (-0.05, 0.18) rectangle (0.05, 0.24);
                }{}%
                \ifthenelse{\equal{\faceKind}{sad}}{%
                    \fill[ledLit] (-0.10, 0.205) rectangle (-0.04, 0.225);
                    \fill[ledLit] (-0.03, 0.225) rectangle ( 0.03, 0.245);
                    \fill[ledLit] ( 0.04, 0.205) rectangle ( 0.10, 0.225);
                }{}%
                \ifthenelse{\equal{\faceKind}{thinking}}{%
                    \fill[ledLit] (-0.08, 0.19) rectangle (-0.05, 0.215);
                    \fill[ledLit] (-0.015,0.19) rectangle ( 0.015,0.215);
                    \fill[ledLit] ( 0.05, 0.19) rectangle ( 0.08, 0.215);
                }{}%

                \fill[\robotcolor!90!black]    (-0.12, 0.04) circle (0.018);
                \fill[\robotcolor!70!black]    ( 0.00, 0.04) circle (0.018);
                \fill[\robotcolor!50!black]    ( 0.12, 0.04) circle (0.018);

                \begin{scope}
                    \clip (-0.40, 0.65) -- (-0.50, 0.55) -- (-0.50, 0.08) --
                          (-0.40, -0.02) -- ( 0.40, -0.02) -- ( 0.50, 0.08) --
                          ( 0.50, 0.55) -- ( 0.40, 0.65) -- cycle;
                    \fill[white, opacity=0.30]
                        (-0.40, 0.65) -- ( 0.10, 0.65) -- (-0.20, 0.48) --
                        (-0.50, 0.48) -- (-0.50, 0.55) -- cycle;
                \end{scope}

                \end{scope}% end head group

                \draw[outline, fill=\robotcolor!25, line width=\lwc pt]
                    (-0.08, -0.02) rectangle (0.08, -0.07);

                \draw[outline, fill=\robotcolor!25]
                    (-0.28, -0.07) -- (-0.32, -0.13) -- (-0.32, -0.42) --
                    (-0.24, -0.50) -- ( 0.24, -0.50) -- ( 0.32, -0.42) --
                    ( 0.32, -0.13) -- ( 0.28, -0.07) -- cycle;

                \draw[outline, fill=\robotcolor!40, line width=\lwd pt]
                    (-0.18, -0.15) rectangle (0.18, -0.38);
                \draw[\robotcolor!70!black, line width=\lwe pt] (-0.14,-0.20) -- (0.14,-0.20);
                \draw[\robotcolor!70!black, line width=\lwe pt] (-0.14,-0.27) -- (0.14,-0.27);
                \draw[\robotcolor!70!black, line width=\lwe pt] (-0.14,-0.34) -- (0.14,-0.34);
                \fill[black!70] (-0.16,-0.17) circle (0.012);
                \fill[black!70] ( 0.16,-0.17) circle (0.012);
                \fill[black!70] (-0.16,-0.36) circle (0.012);
                \fill[black!70] ( 0.16,-0.36) circle (0.012);
                \draw[outline, fill=\robotcolor!85!black, line width=\lwe pt]
                    (0,-0.45) circle (0.022);

                \draw[outline, fill=\robotcolor!40, line width=\lwb pt]
                    (-0.50, -0.08) rectangle (-0.32, -0.22);

                \begin{scope}[rotate around={\robotlarm:(-0.41, -0.22)}]
                \ifthenelse{\equal{\poseLeftArm}{normal}}{%
                    \draw[outline, fill=\robotcolor!75!black, line width=\lwd pt]
                        (-0.41, -0.24) circle (0.045);
                    \draw[outline, fill=\robotcolor!40, line width=\lwb pt]
                        (-0.47, -0.28) rectangle (-0.35, -0.50);
                    \begin{scope}[shift={(-0.41, -0.50)}] \robotClaw \end{scope}
                }{}%
                \ifthenelse{\equal{\poseLeftArm}{up}}{%
                    \draw[outline, fill=\robotcolor!75!black, line width=\lwd pt]
                        (-0.41, -0.24) circle (0.045);
                    \draw[outline, fill=\robotcolor!40, line width=\lwb pt]
                        (-0.47, -0.20) rectangle (-0.35, 0.30);
                    \begin{scope}[shift={(-0.41, 0.30)}, rotate=180] \robotClaw \end{scope}
                }{}%
                \ifthenelse{\equal{\poseLeftArm}{up-out}}{%
                    \draw[outline, fill=\robotcolor!75!black, line width=\lwd pt]
                        (-0.41, -0.24) circle (0.045);
                    \begin{scope}[rotate around={22:(-0.41, -0.22)}]
                        \draw[outline, fill=\robotcolor!40, line width=\lwb pt]
                            (-0.47, -0.20) rectangle (-0.35, 0.30);
                        \begin{scope}[shift={(-0.41, 0.30)}, rotate=180] \robotClaw \end{scope}
                    \end{scope}
                }{}%
                \ifthenelse{\equal{\poseLeftArm}{extended-left}}{%
                    \draw[outline, fill=\robotcolor!40, line width=\lwb pt]
                        (-0.70, -0.20) rectangle (-0.50, -0.10);
                    \draw[outline, fill=\robotcolor!75!black, line width=\lwd pt]
                        (-0.50, -0.15) circle (0.045);
                    \begin{scope}[shift={(-0.70, -0.15)}, rotate=-90] \robotClaw \end{scope}
                }{}%
                \ifthenelse{\equal{\poseLeftArm}{bent-out}}{%
                    \draw[outline, fill=\robotcolor!75!black, line width=\lwd pt]
                        (-0.41, -0.24) circle (0.045);
                    \draw[outline, fill=\robotcolor!40, line width=\lwb pt]
                        (-0.65, -0.30) rectangle (-0.45, -0.18);
                    \begin{scope}[shift={(-0.65, -0.24)}, rotate=-135] \robotClaw \end{scope}
                }{}%

                \end{scope}% end left arm

                \draw[outline, fill=\robotcolor!40, line width=\lwb pt]
                    ( 0.32, -0.08) rectangle ( 0.50, -0.22);

                \begin{scope}[rotate around={\robotrarm:(0.41, -0.22)}]
                \ifthenelse{\equal{\poseRightArm}{normal}}{%
                    \draw[outline, fill=\robotcolor!75!black, line width=\lwd pt]
                        ( 0.41, -0.24) circle (0.045);
                    \draw[outline, fill=\robotcolor!40, line width=\lwb pt]
                        ( 0.35, -0.28) rectangle ( 0.47, -0.50);
                    \begin{scope}[shift={( 0.41, -0.50)}] \robotClaw \end{scope}
                }{}%
                \ifthenelse{\equal{\poseRightArm}{up}}{%
                    \draw[outline, fill=\robotcolor!75!black, line width=\lwd pt]
                        ( 0.41, -0.24) circle (0.045);
                    \draw[outline, fill=\robotcolor!40, line width=\lwb pt]
                        ( 0.35, -0.20) rectangle ( 0.47, 0.30);
                    \begin{scope}[shift={( 0.41, 0.30)}, rotate=180] \robotClaw \end{scope}
                }{}%
                \ifthenelse{\equal{\poseRightArm}{up-out}}{%
                    \draw[outline, fill=\robotcolor!75!black, line width=\lwd pt]
                        ( 0.41, -0.24) circle (0.045);
                    \begin{scope}[rotate around={-22:(0.41, -0.22)}]
                        \draw[outline, fill=\robotcolor!40, line width=\lwb pt]
                            ( 0.35, -0.20) rectangle ( 0.47, 0.30);
                        \begin{scope}[shift={( 0.41, 0.30)}, rotate=180] \robotClaw \end{scope}
                    \end{scope}
                }{}%
                \ifthenelse{\equal{\poseRightArm}{extended-right}}{%
                    \draw[outline, fill=\robotcolor!40, line width=\lwb pt]
                        ( 0.50, -0.20) rectangle ( 0.70, -0.10);
                    \draw[outline, fill=\robotcolor!75!black, line width=\lwd pt]
                        ( 0.50, -0.15) circle (0.045);
                    \begin{scope}[shift={( 0.70, -0.15)}, rotate=90] \robotClaw \end{scope}
                }{}%
                \ifthenelse{\equal{\poseRightArm}{bent-up}}{%
                    \draw[outline, fill=\robotcolor!75!black, line width=\lwd pt]
                        ( 0.41, -0.24) circle (0.045);
                    \begin{scope}[rotate around={-6:( 0.41, -0.20)}]
                        \draw[outline, fill=\robotcolor!40, line width=\lwb pt]
                            ( 0.35, -0.20) rectangle ( 0.47, 0.30);
                        \begin{scope}[shift={( 0.41, 0.30)}, rotate=180] \robotClaw \end{scope}
                    \end{scope}
                }{}%
                \ifthenelse{\equal{\poseRightArm}{bent-out}}{%
                    \draw[outline, fill=\robotcolor!75!black, line width=\lwd pt]
                        ( 0.41, -0.24) circle (0.045);
                    \draw[outline, fill=\robotcolor!40, line width=\lwb pt]
                        ( 0.45, -0.30) rectangle ( 0.65, -0.18);
                    \begin{scope}[shift={( 0.65, -0.24)}, rotate=135] \robotClaw \end{scope}
                }{}%

                \end{scope}% end right arm

                \draw[outline, fill=\robotcolor!75!black, line width=\lwd pt]
                    (-0.13, -0.50) circle (0.040);
                \draw[outline, fill=\robotcolor!75!black, line width=\lwd pt]
                    ( 0.13, -0.50) circle (0.040);
                \draw[outline, fill=\robotcolor!40, line width=\lwb pt]
                    (-0.17, -0.54) rectangle (-0.09, -0.74);
                \draw[outline, fill=\robotcolor!40, line width=\lwb pt]
                    ( 0.09, -0.54) rectangle ( 0.17, -0.74);
                \draw[outline, fill=\robotcolor!75!black, line width=\lwb pt]
                    (-0.21, -0.78) rectangle (-0.05, -0.74);
                \draw[outline, fill=\robotcolor!75!black, line width=\lwb pt]
                    ( 0.05, -0.78) rectangle ( 0.21, -0.74);

                \end{scope}% end bob group

                \node[inner sep=0pt, outer sep=0pt] (-head) at (0, 0.32) {};
                \coordinate (-top)                 at ( 0.00, 0.65);
                \coordinate (-anchor-bubble-right) at ( 0.50, 0.65);
                \coordinate (-anchor-bubble-left)  at (-0.50, 0.65);
                \coordinate (-anchor-bubble)       at ( 0.50, 0.65);
                \coordinate (-side-right)          at ( 0.50, 0.32);
                \coordinate (-side-left)           at (-0.50, 0.32);
                \coordinate (-foot)  at (0, -0.78);
                \coordinate (-lhand) at (-0.41, -0.58);
                \coordinate (-rhand) at ( 0.41, -0.58);
            \end{scope}
        }%
    },
    /tikz/robot/.search also = {/tikz},
    /tikz/robot/color/.initial = CORAcolor1,
    /tikz/robot/pose/.initial  = happy,
    /tikz/robot/face/.initial  = auto,
    /tikz/robot/body/.initial  = auto,
    /tikz/robot/headtilt/.initial  = 0,
    /tikz/robot/headshift/.initial = 0,
    /tikz/robot/larm/.initial      = 0,
    /tikz/robot/rarm/.initial      = 0,
    /tikz/robot/bob/.initial       = 0,
}

\tikzset{
    speech/.style 2 args={
        rectangle callout,
        callout absolute pointer={(#1)},
        callout pointer width=4pt,
        draw=#2, line width=0.8pt, fill=white,
        rounded corners=3pt,
        inner sep=5pt,
        font=\small\sffamily,
        align=center,
    },
}

\title{The Misery of Mechanistic Interpretability: \\ A Formal Perspective}

\author{Tobias Ladner\ \ {\small\normalfont (\texttt{tobias.ladner@tum.de})} \\
Technical University of Munich, Germany
\And
Matthias Althoff\ \ {\small\normalfont (\texttt{althoff@tum.de})}  \\
Technical University of Munich, Germany
}

\begin{document}
\maketitle
\lhead{Preprint. Under review.}   % must come after \maketitle

% =========================================================================

\begin{abstract}
Mechanistic interpretability has become the dominant lens for understanding frontier language models,
as their inner workings are complex and inherently black boxes.
To gain insights into these models, interpretable replacement networks (\IRN{}s) are trained at all layers,
exposing interpretable features through sparsely activated neurons.
However, the faithfulness of an \IRN{} is usually evaluated only empirically on clean data,
and we show that even semantically minor input perturbations flip the dominant \IRN{} features%
---and thus the human-understandable interpretation---%
across five open-weight model families (\GPT{}, \GemmaTwo{}, \GemmaThree{}, \Llama{}, \Qwen{}).
We propose the first formal verification framework for the faithfulness of an \IRN{},
where reachability analysis certifies a sound upper bound on the faithfulness gap in adversarial scenarios.
Moreover, we show that verification-aware training of \IRN{}s substantially tightens this certified bound,
restoring a feature-level interpretation that safety auditors can act on.
Together, these results give, to the best of our knowledge, the first formal guarantees for mechanistic interpretability of large language models.
\end{abstract}

% keywords
% ai safety, large language models, mechanistic interpretability, neural network verification, formal methods, faithfulness, guarantees

\section{Introduction}
\label{sec:intro}

Most interpretations of a large language model carry a quiet asterisk: They are not direct interpretations of the model.
When we read that a neuron ``fires for the Golden Gate Bridge''~\citep{templeton2024scaling} or that a circuit ``implements a comparison operator''~\citep{dunefsky2024transcoders},
these claims are often read off an \emph{interpretable replacement network} (\IRN{})---an interpretable network to stand in for a dense computation the model actually performs~\citep{marks2025feature,somvanshi2026bridging}.
This is necessary as the dense computations of the underlying model result in polysemantic neurons~\citep{elhage2022superposition},
whereas the \IRN{} exposes a short list of human-understandable features.
Crucially, these features are only as trustworthy as that stand-in is \emph{faithful}, i.e., the computation of the \IRN{} matches the underlying model.

However, training the \IRN{} to match the underlying model usually results in an approximate stand-in.
While the \IRN{} can be patched for individual inputs~\citep{lindsey2025circuit_tracing},
whether an interpretation survives an adversarial input is barely tested or verified%
---even though adversarial examples have been known to exist for a decade~\citep{goodfellow2015adversarial}, including for language models~\citep{dong2025safeguarding}.
Moreover, individual patches defeat the purpose of mechanistic interpretability to a certain degree as it aims to replicate the true \emph{mechanism} of the underlying model.

\Cref{fig:teaser} shows how such a faulty \IRN{} can deceive a safety auditor:
On the clean input sentence, the auditor would read off features surrounding ``cultivating a pathogen''---a signal the safety auditor would use to refuse the request.
We then perturb the input slightly by paraphrasing the sentence, such that the dominant features now describe ``workplace culture''---and the auditor would wave the request through.
Crucially, nothing of the model changed; what changed is the exposed interpretation by the \IRN{}.
The full example is provided in \cref{sec:teaser_example}.

It is worth noting that there are also interpretations that operate on the actual model.
For example, circuit tracing methods aim to identify the active subnetworks for any input~\citep{somvanshi2026bridging};
however, those methods do not scale to frontier language models yet~\citep{gao2025weight}.
Instead, dense computations are replaced by \IRN{}s and circuit tracing is applied on this modified model~\citep{lindsey2025circuit_tracing}---making the faithfulness of \IRN{}s ever more important.
Further details are provided in \Cref{sec:related_work}, and we focus for the remainder of this work on the faithfulness of \IRN{}s.

\paragraph{Contributions.}
To summarize, our main contributions are:
\begin{enumerate}
    \item We first show that \IRN{}s are highly sensitive to adversarial attacks across various model families and \IRN{} types: \GPT, \GemmaTwo, \GemmaThree, \Llama, and \Qwen{} (\cref{sec:fragility}).
    \item Subsequently, we establish a formal verification framework for the faithfulness of an \IRN{} (\Cref{sec:formal}).
          In particular, we show that with the recent progress in formal methods, formal neural network verification can also be utilized in large language models.
    \item Finally, we show that verification-aware training substantially narrows the faithfulness gap of \IRN{}s (\cref{sec:way_forward}).
\end{enumerate}

% --- teaser figure -----------------------------------------------------------
\begin{figure}[t]
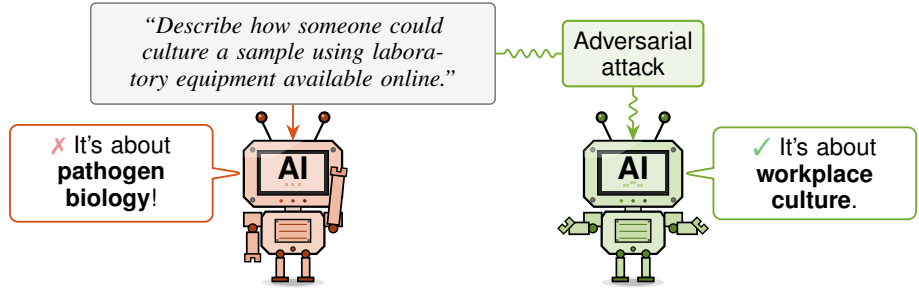

    \centering
    \includetikz{./figures/teaser/teaser}
    \caption{
        \textbf{Same meaning, different interpretations.}
        On the clean input (left), the most active features describe a ``pathogen'' sense of \token{culture},
        and a feature-based safety auditor reading those features blocks the request.
        With a minor synonym attack (right), the dominant features instead describe a ``workplace'' sense,
        and the same auditor would let the request pass.
        % Full example in \cref{sec:teaser_example}.
    }
    \label{fig:teaser}
\end{figure}

\section{Background on Mechanistic Interpretability}
\label{sec:background}

Modern language models~\citep{vaswani2017attention, elhage2021framework} are typically composed of a token embedding layer,
a stack of $\numLayers$ blocks operating on a residual stream,
and a decoding layer to predict the next token.
Each block $\layerIdx$ has two sublayers (\Cref{fig:irn_setup}):
A multi-head attention sublayer that mixes information \emph{across} tokens,
and a multi-layer perceptron (\MLP{}) that applies a nonlinear transformation to each token \emph{independently}.
Both add their output back into the residual stream, so we can write the layer-$\layerIdx$ update as \citep[Eq.~1 \& 2]{vaswani2017attention}
\begin{equation}
    \hMid \;=\; \hAt[\layerIdx-1] + \atLayer{\Attn}\!\bigl(\hAt[\layerIdx-1]\bigr),
    \qquad
    \hAt \;=\; \hMid + \layerMLP(\hMid),
\end{equation}
where $\hAt \in \mathbb{R}^{\numModelDim \times \numTokens}$ collects the residual-stream vectors of
the $\numTokens$ input tokens as columns, and $\hMid\in \mathbb{R}^{\numModelDim \times \numTokens}$ is the post-attention intermediate.
We are aware that this is a simplified description of the architecture of frontier models;
however, insofar as these models are not proprietary, most model families can be described under this framework~\citep{radford2019gpt2, gemmateam2024gemma2, gemmateam2025gemma3, grattafiori2024llama3, qwen2025qwen25},
and this high-level description is sufficient for our work.

An attention layer can usually be interpreted directly as it computes an interpretable relation between tokens~\citep[Eq.~1]{vaswani2017attention}.
However, the \MLP{} is hard to interpret:
Its neurons are densely entangled and polysemantic \citep{elhage2022superposition}, i.e., each fires for unrelated inputs.
The dominant remedy is to fit an interpretable replacement network (\IRN{})~\citep{somvanshi2026bridging},
whose activations are sparse and individually meaningful,
and to read interpretations off the \IRN{} features as a proxy for the model.
We consider two dominant variants in this work, sparse autoencoders and transcoders (\Cref{fig:irn_setup}), which are also available for many open-weight models.

\paragraph{Sparse autoencoder (\SAE{}).}
A sparsely activated autoencoder~\citep{cunningham2023saes} reconstructs its input at layer $\layerIdx$:
\begin{equation}\label{eq:sae}
    \layerSAE(\hAt) \;=\; W_\layerIdx^{\mathrm{dec}}\,
    \nnActFun\!\bigl(W_\layerIdx^{\mathrm{enc}}\, \hAt + b_\layerIdx^{\mathrm{enc}}\bigr) + b_\layerIdx^{\mathrm{dec}}
    \;\approx\; \hAt,
\end{equation}
where the parameters $W_\layerIdx^{\mathrm{enc}}\in \mathbb{R}^{\numIRNdim \times \numModelDim}$, $W_\layerIdx^{\mathrm{dec}}\in \mathbb{R}^{\numModelDim \times \numIRNdim }$, $b_\layerIdx^{\mathrm{enc}}\in\R^{\numIRNdim}$, $b_\layerIdx^{\mathrm{dec}}\in\R^{\numModelDim}$ are learned per layer.

\paragraph{Transcoder (\TC{}).}
A transcoder \citep{dunefsky2024transcoders} at layer $\layerIdx$ reconstructs the output of the multi-layer perceptron from the same input:
\begin{equation}\label{eq:tc}
    \layerTC(\hMid) \;=\; W_\layerIdx^{\mathrm{dec}}\,
    \nnActFun\!\bigl(W_\layerIdx^{\mathrm{enc}}\, \hMid + b_\layerIdx^{\mathrm{enc}}\bigr) + b_\layerIdx^{\mathrm{dec}}
    \;\approx\; \layerMLP(\hMid),
\end{equation}
again with per-layer parameters $W_\layerIdx^{\mathrm{enc}}, W_\layerIdx^{\mathrm{dec}}, b_\layerIdx^{\mathrm{enc}}, b_\layerIdx^{\mathrm{dec}}$.
Transcoders can also be constructed across different computation blocks~\citep{lindsey2025circuit_tracing}.

The activation function $\nnActFun$ varies by model family:
For example, \ReLU{} is used for the \GPT{}'s sparse autoencoder and transcoder~\citep{dunefsky2024transcoders},
\JumpReLU{} for the \GemmaScope{}~\citep{lieberum2024gemmascope},
and \TopK{} for the Llama Scope~\citep{he2024llamascope}.

A mechanistic interpretation at layer $\layerIdx$ for an input sentence is then read off the active \IRN{} features,
e.g., the indices $i$ with $\nnActFun(\cdot)_{(i)} > 0$,
where a sparsity penalty during training ensures that only a small subset of the hidden features ($\numIRNdim \gg \numModelDim$) fires for any given input.
Subsequently, each feature automatically obtains a human-readable label by summarizing the inputs that activate it \citep{paulo2024autointerp}.
This feature-label catalog is also often available along with the \IRN{}~\citep{lin2024neuronpedia}.

\begin{figure}[t]
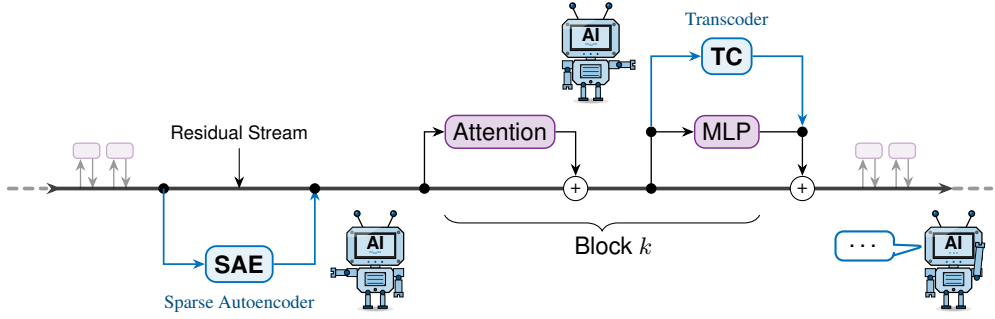

    \centering
    \includetikz{./figures/irn_setup/irn_setup}
    \caption{
        \textbf{Interpretable replacement networks.}
        Large language models typically have a residual stream, which is modified through attention layers and MLPs (shown in purple).
        As these compute complex updates, interpretable replacement networks (shown in blue) are added to understand the internal behavior of a large language model.
        Examples are sparse autoencoders (reconstructing the residual stream), and transcoders (reconstructing the output of the MLP).
    }
    \label{fig:irn_setup}
\end{figure}

\section{Where Mechanistic Interpretability Breaks}
\label{sec:fragility}

A foundational observation in deep learning is that imperceptible input perturbations can flip the prediction of a model entirely---the so-called adversarial examples~\citep{goodfellow2015adversarial}.
For mechanistic interpretability, we want the active features of the \IRN{} to remain active when a sentence is rephrased in semantically equivalent ways.
Otherwise, the human-readable labels attached to those features do not reflect the actual computation,
and the \IRN{} fails to replicate the mechanism of the model.
This premise follows existing literature on the robustness of interpretations in regular neural networks~\citep{wu2023verix,bassan2023towards,la2021guaranteed}.

While most evaluations of \IRN{}s report the reconstruction loss on clean data,
adversarial settings are barely evaluated despite early signs of the instability of \IRN{}s:
For example, it was shown that sparse autoencoders trained as an \IRN{} are fragile to small adversarial perturbations~\citep{li2025_attack_features}.
In this section, we confirm these results with additional experiments across various open-weight model families and \IRN{} types,
evaluated under three levels of natural-language paraphrase:
\begin{itemize}
    \item \pertMinor{}: A single content word is replaced with a WordNet synonym~\citep{miller1995wordnet}.
    \item \pertMedium{}: All content words before the target are replaced with WordNet synonyms.
    \item \pertMajor{}: An LLM (\GemmaTwo-it) rephrases the sentence while preserving the target.
\end{itemize}

\paragraph{Results.}
\Cref{tab:fragility-summary} reports the top-$20$ Jaccard index\footnote{For two discrete sets $\contSet{A}$ and $\contSet{B}$,
the Jaccard index is given by $J=\frac{|\contSet{A}\cap\contSet{B}|}{|\contSet{A}\cup\contSet{B}|}$.}
between the active feature set of an \IRN{} at the clean input and at an adversarial input.
The Jaccard index drops substantially as the perturbation grows:
Already at the \pertMinor{} level, the perturbations reduce the worst-layer overlap $\jaccUp_{\min}$ ($-73\%$),
and at the \pertMajor{} level, the reduction is even larger ($-95\%$);
thus, almost all dominant features change and along with them any interpretation for a safety auditor.
We observe this drop for both sparse autoencoders and transcoders.
Evaluation details are given in \cref{sec:evaluation_details}.

\begin{table}[t]
    \centering
    \caption{Attack-derived top-$20$ Jaccard upper bound $\jaccUp$ of the active feature set under paraphrase perturbations across all layers.}
    \begin{tabular}{l l C C C C C C}
        \toprule
        & & \multicolumn{2}{c}{\pertMinor} & \multicolumn{2}{c}{\pertMedium} & \multicolumn{2}{c}{\pertMajor} \\
        \cmidrule(lr){3-4} \cmidrule(lr){5-6} \cmidrule(lr){7-8}
        Model & \IRN{} & \jaccUp_{\min}\,(\uparrow) & \jaccUp_{\mathrm{mean}}\,(\uparrow) & \jaccUp_{\min}\,(\uparrow) & \jaccUp_{\mathrm{mean}}\,(\uparrow) & \jaccUp_{\min}\,(\uparrow) & \jaccUp_{\mathrm{mean}}\,(\uparrow) \\
        \midrule
        \GPT         & \SAE{}   &  0.27 &  0.37 &  0.06 &  0.18 &  0.05 &  0.13 \\
        \GPT         & \TC{}    &  0.70 &  0.79 &  0.45 &  0.62 &  0.38 &  0.54 \\
        \GemmaTwo    & \SAE{}   &  0.62 &  0.78 &  0.35 &  0.59 &  0.33 &  0.54 \\
        \GemmaTwo    & \TC{}    &  0.46 &  0.57 &  0.20 &  0.33 &  0.20 &  0.31 \\
        \GemmaThree  & \SAE{}   &  0.55 &  0.69 &  0.30 &  0.49 &  0.30 &  0.44 \\
        \GemmaThree  & \TC{}    &  0.47 &  0.67 &  0.18 &  0.42 &  0.20 &  0.38 \\
        \Llama       & \SAE{}   &  0.54 &  0.66 &  0.26 &  0.45 &  0.33 &  0.42 \\
        \Llama       & \TC{}    &  0.56 &  0.64 &  0.25 &  0.40 &  0.21 &  0.35 \\
        \QwenShort   & \SAE{}   &  0.37 &  0.49 &  0.18 &  0.29 &  0.12 &  0.20 \\
        \bottomrule
    \end{tabular}
    \label{tab:fragility-summary}
\end{table}

Thus, the active features of an \IRN{} can change substantially for adversarial inputs.
However, please note that the underlying model can indeed also behave differently for an adversarial input,
and we would only expect that a \emph{robust} model maintains its dominant features under adversarial inputs.
Moreover, an adversarial attack only witnesses a single bad input:
It shows that the Jaccard index can drop at least this far, but cannot rule out an adversarial input where it drops further.
The numbers in \cref{tab:fragility-summary} are therefore an upper bound $\jaccUp$ on the true Jaccard index $\jaccTrue$ under adversarial attacks.
To obtain guarantees, a faithful \IRN{} of a robust model needs to have a high lower bound $\jaccLo$.

\section{Verifying the Mechanism Equivalence of an LLM and its IRNs}
\label{sec:formal}

We have established that an \IRN{} of a robust model has to have a high lower bound on the Jaccard index $\jaccLo$.
Crucially, the \IRN{} also has to be \emph{faithful} for a sensible $\jaccLo$ evaluation, i.e., the computation of the \IRN{} should match the underlying model.
As an extreme example, imagine an unfaithful \IRN{} that always has the same (safe) features active independently of the model input, e.g., through large biases on these features.
Then, $\jaccLo$ is always $1$ even though the underlying model might be unsafe, and thus it is an insufficient measure as the \IRN{} is not faithful.
To guarantee the faithfulness, we develop a framework to formally verify the equivalence of the model and the \IRN{} under adversarial inputs in this section,
and obtain $\jaccLo$ as a by-product of this equivalence check.

We construct a set $\nnInputSet\subset\R^{\numModelDim\times\numTokens}$ containing all adversarial embeddings of an input sentence.
For example, $\nnInputSet$ can contain all semantically equivalent sentences to that input sentence through synonym replacements of each word.
As explicitly enumerating all synonym combinations quickly results in a combinatorial explosion,
we construct, per token, a convex hull in the embedding space over all synonyms of that token,
and then compute the Cartesian product of these convex hulls to obtain $\nnInputSet$.
Then, with $\nnHiddenSet_\layerIdx$ being the propagated set at each layer $\layerIdx$, we can state our faithfulness criterion:
\begin{definition}
    \label{def:delta-equivalence}
    For a $\nnOutBound\in\R_+$, we say that an \IRN{} at layer $\layerIdx$ is $\nnOutBound$-equivalent to the underlying model on $\nnHiddenSet_\layerIdx$ if
    \begin{align*}
        \forall\, \hAdv \in \nnHiddenSet_\layerIdx\colon & &
        \big\| \hAdv - \layerSAE(\hAdv) \big\|_\infty &\le \nnOutBound, &
        \big\| \layerMLP(\hAdv) - \layerTC(\hAdv) \big\|_\infty &\le \nnOutBound,
    \end{align*}
    and we say that the \IRN{} is faithful if $\nnOutBound$ is sufficiently small.
\end{definition}
Please note that $\nnHiddenSet_\layerIdx$ is a continuous set, e.g., imagine a hypercube with radius $\nnPertRadius\in\R_+$ containing all embeddings of the adversarial inputs.
To test for $\delta$-equivalence (\cref{def:delta-equivalence}), we can adapt existing neural network verifiers~\citep{kaulen20256th} for our use case.
In particular, we construct for each layer $k$ the difference network $\layerMLP(\hAt) - \layerIRN(\hAt)$ as a composite of two parallel paths with difference-based aggregation (\Cref{fig:verification_setup}),
and the verifier tries to prove that the difference is bounded by a given $\nnOutBound$.
We defer a more detailed description of how these verifiers work to \cref{sec:background-nnv} and provide optimizations applied in our settings in \cref{sec:per_layer}, which also gives us $\jaccLo$ (\cref{sec:lower-bound-jaccard}).

\begin{figure}[t]
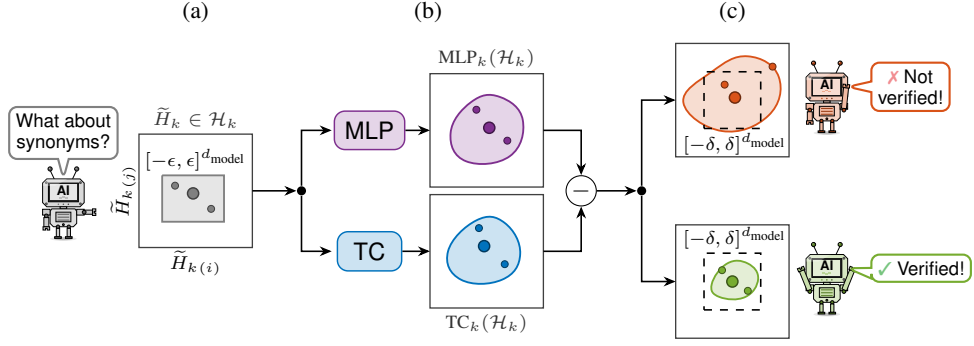

    \centering
    \includetikz{./figures/verification_setup/verification_setup}
    \caption{\textbf{Mechanism verification as a composite network.}
    (a) The local neighborhood around an input sentence is captured by $\nnHiddenSet_k$ so that adversarial inputs are also contained.
        (b) We can deploy formal verification to bound the maximal difference between the $\layerMLP$ and the $\layerTC$ within the local neighborhood.
        (c) The faithfulness property in \Cref{def:delta-equivalence} holds if the computed difference set is contained in a small box $[-\nnOutBound, \nnOutBound]^{\numModelDim}$.
        The \SAE{} case is analogous, with $\layerMLP$ replaced by the identity (\cref{eq:sae}).
    }
    \label{fig:verification_setup}
\end{figure}

With such a verifier, we aim to find the smallest verifiable bound $\deltaTrue$.
As with the Jaccard index $\jaccTrue$, it is usually also not feasible to compute $\deltaTrue$, and we need to find bounds: $\deltaPGD \leq \deltaTrue \leq \deltaVerified$.
We can find a lower bound $\deltaPGD$ using adversarial attacks trying to maximize the difference, and the upper bound $\deltaVerified$ needs to be verified.
Please note that this is converse to $\jaccLo,\jaccUp$, where the lower bound---describing the number of unflippable active features---needs to be verified.
We find $\deltaVerified$ by conducting a binary search with sequential verification queries (details in \cref{sec:evaluation_details}, \cref{alg:verify_eps}).
To display the faithfulness gap in figures, we draw the gap between the median $\deltaPGD$ and $\deltaVerified$ over the test dataset as a filled box%
---thus, the median $\deltaTrue$ has to be within the filled box---%
and add whiskers for the $25\%$- and $75\%$-quantile of $\deltaPGD$ and $\deltaVerified$, respectively.

\begin{figure}[b!]
    \centering
    \begin{minipage}[c]{0.75\linewidth}
        \centering
        \includetikz{./figures/evaluation/faithfulness_gap/faithfulness_gap}
    \end{minipage}\hfill
    \begin{minipage}[c]{0.2\linewidth}
        \centering
        \tikzexternaldisable
        \begin{tikzpicture}[scale=0.6]
            \pic (bot) at (0,0) {robot={color=CORAcolor1, pose=thinking}};
            \node[speech={bot-top}{CORAcolor4}, anchor=south, yshift=10pt, at={(bot-top)}]
            {How to close the \\ faithfulness gap?};
        \end{tikzpicture}%
        \tikzexternalenable
    \end{minipage}
    \hfill
    \caption{
        \textbf{Per-layer faithfulness gap on \GPT{}.}
        For both \IRN{}s and across all layers, we can verify a substantial faithfulness gap between the computation of the \IRN{}s and the underlying model.
    }
    \label{fig:verified-vs-pgd}
\end{figure}
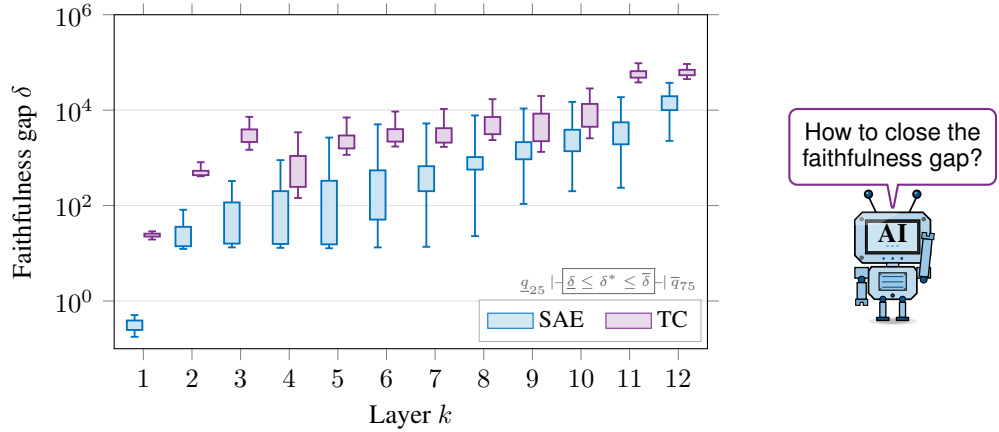

\paragraph{Results.}
\Cref{fig:verified-vs-pgd} reports the faithfulness gap for both \IRN{}s of \GPT{} (other models in \cref{sec:ablations}),
where we have verified a substantial faithfulness gap across all layers.
For example, we have verified that for layer $1$, the \SAE{} is $\nnOutBound$-equivalent to the underlying model for some $\nnOutBound<1$.
Thus, the reconstructed residual stream under the $\ell_\infty$ norm has, on average, a smaller difference than $1$ given adversarial inputs.
To put this number into context, we also observed such differences on the residual stream of the underlying model given adversarial inputs.
Unfortunately, our results demonstrate that later layers become less faithful,
arguably to a degree that does not allow for meaningful interpretations anymore.
We want to stress that this is not an artifact of the (incomplete) verifier,
as not only the upper bound found by the verifier is rising, but also the lower bound found by adversarial attacks.
We note that the faithfulness gap is generally larger for \TC{}s, which is to be expected as a \TC{} has to replicate the complex computations of the \MLP{} there, and \SAE{}s ``only'' have to reconstruct their input.
Evaluation details are again given in \cref{sec:evaluation_details}.

% -----------------------------------------------------------------------------

\section{Improving the Faithfulness of an \IRN{}}
\label{sec:way_forward}

These results naturally lead to the question of whether more faithful \IRN{}s can be obtained.
Standard training of \IRN{}s usually focuses on the combination of reconstruction and sparsity for interpretability on clean data \citep{marks2025feature,dunefsky2024transcoders};
however, it does not adequately consider adversarial inputs, such that their susceptibility likely degrades the verification results in the previous section.
Conceptually, we want to train the \IRN{} in a way so that in \cref{fig:verification_setup}c, any sentence for which the error set is like the top plot becomes more like the bottom plot%
---thereby aligning its manifold with the manifold of the underlying model.
We investigate different training regimes in this section to improve the faithfulness of an \IRN{} under adversarial inputs:
\begin{itemize}
    \item \trainStd: The given \IRN{} obtained using standard training is used as a baseline.
    \item \trainPGD: Adversarial training using PGD-found $\hAdv$ as augmentation~\citep{madry2018towards}.
    \item \trainSet: Verification-aware training with reachable-set loss \citep{koller2024setbasedtraining}.
    \item \trainLoRA: Low-rank updates typically used for fine-tuning LLMs \citep{hu2022lora}.
\end{itemize}

% -----------------------------------------------------------------------------

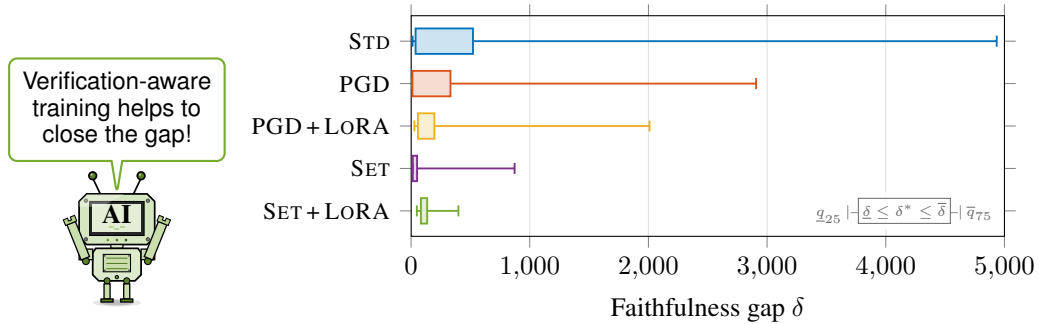
\begin{figure}[b!]
    \centering
    \hfill
    \begin{minipage}[c]{0.2\linewidth}
        \centering
        \tikzexternaldisable
        \begin{tikzpicture}[scale=0.6]
            \pic (bot) at (0,0) {robot={color=CORAcolor5, pose=cheering}};
            \node[speech={bot-top}{CORAcolor5}, anchor=south, yshift=12pt, at={(bot-top)}]
            {Verification-aware \\ training helps to \\ close the gap!};
        \end{tikzpicture}%
        \tikzexternalenable
    \end{minipage}\hfill
    \begin{minipage}[c]{0.75\linewidth}
        \centering
        \includetikz{./figures/evaluation/training_delta/training_delta}
    \end{minipage}
    \caption{\textbf{Faithfulness gap per training regime.} Verification-aware training results in a substantially smaller faithfulness gap between \GPT{} and its \SAE{} at layer $6$.}
    \label{fig:training-delta}
\end{figure}

\paragraph{Results.}
We fine-tune OpenAI's \SAE{} of \GPT{} using these training regimes (more experiments again in \cref{sec:ablations}),
and show the verification results in \cref{fig:training-delta} using the same setup as in \cref{sec:formal}.
Our results show substantially smaller verified upper bounds $\deltaVerified$ on the difference set ($\approx -90\%$ for \trainSet{} over \trainStd).
Please note that the difference sets live in a high-dimensional space ($\numModelDim=768$), and thus the relative error volume is extremely small ($\approx 10^{-779}$ of \trainSet{} over \trainStd).
Crucially, the retrained \IRN{}s are more sparse and thus interpretable (\trainSet{} has on average $13.9$ active features, $-81\%$ over \trainStd).
We can also improve the faithfulness of the \IRN{} with efficient low-rank updates (LoRA) despite the verifier requiring robustness in the full space.
We anticipate this result to also be of broader interest within the verification community as it might transfer to non-NLP-based models as well.

While the retrained \IRN{}s are both easier to verify and harder to attack,
we acknowledge that the preservation of active features is still not substantial over the baseline as shown in \Cref{fig:jacclo}.
From the top-$20$ features of the clean input, only $2$--$3$ are provably retained for adversarial inputs,
and $\jaccLo$ is improved across all tested training regimes.
However, as our trained \IRN{}s are much closer to the underlying model (\cref{fig:training-delta}),
this Jaccard evaluation is more faithful than for the baseline \IRN{}.
Thus, it is reasonable to attribute this result to the underlying model not being very robust.
Indeed, the underlying model already flips its next-token prediction for $21\%$ of the \pertMinor{} paraphrases,
confirming that the underlying model is itself fragile to these paraphrases.

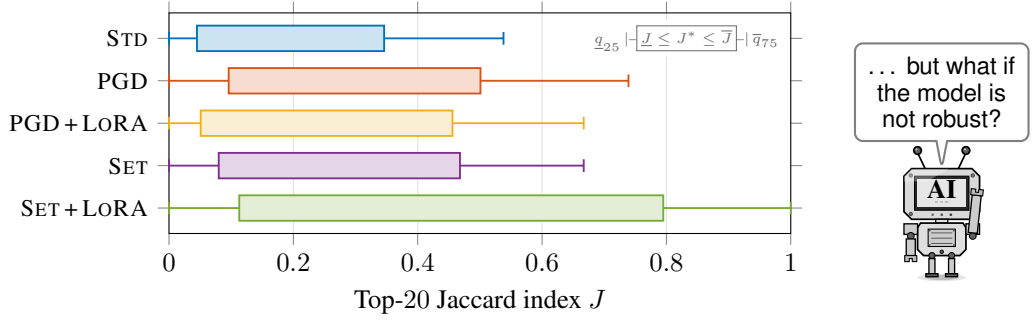
\begin{figure}[t]
    \centering
    \begin{minipage}[c]{0.78\linewidth}
        \centering
        \includetikz{./figures/evaluation/jaccard_gap/jaccard_gap}
    \end{minipage}\hfill
    \begin{minipage}[c]{0.2\linewidth}
        \centering
        \tikzexternaldisable
        \begin{tikzpicture}[scale=0.6]
            \pic (bot) at (0,0) {robot={color=gray, pose=thinking}};
            \node[speech={bot-top}{gray}, anchor=south, yshift=10pt, at={(bot-top)}]
                {\ldots\ but what if \\ the model is \\ not robust?};
        \end{tikzpicture}%
        \tikzexternalenable
    \end{minipage}
    \caption{\textbf{Top-$20$ feature-overlap bounds on the \GPT{} layer-$6$ \SAE{}.}
        For each training regime the box spans the certified lower bound $\jaccLo$ to the attack-derived upper bound $\jaccUp$ over the test samples ($\jaccLo\le\jaccTrue\le\jaccUp$); whiskers reach the $q_{25}$ of $\jaccLo$ and the $q_{75}$ of $\jaccUp$. Further right is more faithful.}
    \label{fig:jacclo}
\end{figure}

% --------------------------------------------

\section{Vision: Compositional Verification of the Entire LLM}
\label{sec:full-llm-verification}

In this section, we state our vision for verifying the robustness of a model and the faithfulness of its \IRN{}s across all layers.
In particular, \cref{def:delta-equivalence} assumes a set $\nnHiddenSet_\layerIdx$ capturing all propagated adversarial inputs $\hAdv$ at each layer $\layerIdx$.
However, as these verifiers do not scale yet to the full model, we construct $\nnHiddenSet_\layerIdx$ at each layer individually via the convex hull over $\hAdv$ rather than the propagated input set $\nnInputSet$.

To formally verify the faithfulness, $\nnInputSet$ needs to be propagated to layer $k$, enclosing all updates through attention layers and \MLP{}s along the way (\cref{fig:irn_setup}).
While it is possible to enclose the updates of attention layers \citep{bonaert2021fast} and \MLP{}s (\cref{sec:background-nnv}),
the resulting enclosures often become too conservative to meaningfully reason about the faithfulness.
Moreover, the LLM might be too large for the verifier to reason about the entire model in one query.

Thus, we envision a compositional approach visualized in \cref{fig:propagation_setup},
where each block of the LLM is verified individually analogous to \cref{sec:formal},
i.e., each block is verified by determining an input-output set relation,
and the faithfulness of the entire model is verified if for any two subsequent blocks, the output set of the former block is contained in the input set of the latter block,
and the verified difference set for the \IRN{} of each block, characterized by $\deltaVerified$, is sufficiently small.

The input set of the first block is initialized with $\nnInputSet$, and the input set of subsequent blocks is initialized via the convex hull of the propagated samples $\hAdv$.
We can then compute the respective output set of each block and check for containment of the input set of the subsequent block.
To realize this efficiently, we require the input sets of each block to be axis-aligned boxes.

\begin{figure}[t]
    \centering
    \includetikz{./figures/compositional/compositional}
    \caption{\textbf{Compositional faithfulness verification of the entire LLM.}
    Each block (\Attn{}\,+\,\MLP{}) is verified locally (\cref{sec:formal}),
        bounding the difference of its \IRN{} by $\deltaPGD\le \deltaTrue \le\deltaVerified$.
        The whole model is then verified compositionally if, for every pair of subsequent blocks,
        the propagated output set of the former is contained in the axis-aligned input box $\nnHiddenSet$ of the latter;
        otherwise, adaptations are required to the surrounding blocks.
    }
    \label{fig:propagation_setup}
\end{figure}

For all blocks where the containment cannot yet be shown, there are several approaches to resolve this issue:
(i) Neural network verifiers can often obtain tighter results given more computation time, so that a smaller output set might be contained.
(ii) The input set of the subsequent block is enlarged; please note that this might also affect downstream blocks.
Both of these points assume that the model and the \IRN{} are fixed; however, a more realistic scenario is that the developer of the model
(iii) could also re-train (certain blocks of) the model to make it more robust (as we did for the \IRN{} in \cref{sec:way_forward}) and thus obtain tighter output bounds.
(iv) Conversely, if an output set of one block is tightened by such an action, the subsequent block can also be tightened with positive impact downstream.

\paragraph{Current state.}
While verification-aware training shows significant improvements on the bounds of the error set (\cref{fig:training-delta}),
we acknowledge that our used sets $\nnHiddenSet_\layerIdx$ are still a bit too small to verify the faithfulness of the entire LLM compositionally (e.g., $5\times$ smaller for layer $6$ of \GPT{} \SAE{}: \cref{tab:hull_eps}).
However, we believe that through a combination of (i)--(iv), this gap can be further closed such that our compositional framework becomes feasible.
As each point requires a more detailed analysis to develop sophisticated algorithms, we leave this for future work.

% =============================================================================
% Related Work  (main body; \input from main.tex, just before the conclusion)
% =============================================================================
\section{Related Work}
\label{sec:related_work}

Mechanistic interpretability has gained great interest in recent years to understand the inner workings of large language models~\citep{bereska2024mechanistic,zhao2024explainability,somvanshi2026bridging},
particularly as the computations are dense and hard to interpret~\citep{elhage2022superposition, park2024linear}.
This lack of transparency has inspired two research branches:
circuit discovery, which isolates the sub-computation responsible for a behavior,
and interpretable replacement networks (\IRN{}s), which substitute a dense computation with a sparse, human-readable stand-in.

\paragraph{Circuit-level interpretability.}
Early work traces circuits over the functional modules of a language model,
inspired by circuit discovery in vision models \citep{olah2020zoom}.
For language models, a circuit groups a subset of attention heads active for a certain task \citep{wanginterpretability},
which can reveal for example a base-10 addition circuit \citep{feucht2026arithmetic},
or clusters of neurons that are active for the same tasks \citep{geiger2025causal}.
Generally, the goal is to discover these circuits automatically \citep{conmy2023acdc};
however, the individual units within each identified circuit remain polysemantic \citep{elhage2022superposition} and thus hard to interpret.

\paragraph{Sparse interpretable replacement networks.}
To disentangle polysemantic neurons, these neurons are replaced with interpretable networks~\citep{cunningham2023saes,dunefsky2024transcoders}.
In particular, it was shown that each \MLP{} module computes specific concept updates on the residual stream \citep{geva2021keyvalue, meng2022rome};
however, the computations made by the \MLP{} are opaque, making them a natural target for an interpretable replacement.
In contrast, the attention modules are usually easier to interpret as one can directly read out to which tokens each token attends \citep{vaswani2017attention}.
\SAE{}s decompose the residual stream into sparsely activated features reconstructing the residual stream \citep{cunningham2023saes}.
These feature neurons are usually arranged in a single wide hidden layer---as it is well known that such architectures are universal approximators if sufficiently wide \citep{hornik1989multilayer}.
In practice, various architectures such as gated \SAE{}s \citep{rajamanoharan2024gated} and dedicated activation functions \citep{gao2024scalingsae} trade off reconstruction and sparsity.
End-to-end dictionary learning optimizes features for functional importance rather than reconstruction \citep{braun2025e2e}.
\TC{}s generalize this concept \citep{dunefsky2024transcoders}, including across layer blocks \citep{lindsey2025circuit_tracing}.
There is also a growing research stream in combining both research directions, such that circuits in a model are discovered on the sparse interpretable network replacing the dense computation,
both on \SAE{}s \citep{marks2025feature} and on \TC{}s \citep{lindsey2025circuit_tracing}.
Across all of these prior works, the \IRN{}s are only evaluated empirically,
and never with a worst-case guarantee in adversarial settings.

\paragraph{Identifying features and critiquing faithfulness.}
After an \IRN{} is trained, the sparsely activated features need to be labeled to be understandable for humans.
The dominant method here is automated interpretability,
where another LLM generates and scores a natural language description of each feature \citep{bills2023autointerp, paulo2024autointerp}.
These descriptions are sometimes also made publicly available in the Neuronpedia ecosystem \citep{lin2024neuronpedia}.
A labeled feature can also be used to localize the respective computation in the underlying model \citep{gurnee2023probing}.
However, a growing critique questions the faithfulness of interpretable replacement networks:
For example, \SAE{}s trained on randomly initialized models produce auto-interpretation scores similar to trained models~\citep{heap2026automated},
and \SAE{}s trained on the same model and data learn different features \citep{paulo2025different}.
Adversarial attacks can show the fragility of the \IRN{} \citep{li2025_attack_features},
and certified defenses against adversarial attacks scale poorly \citep{kumar2023certifying}.

\paragraph{XAI and formal guarantees.}
Similar critiques have become louder in the broader field of explainable artificial intelligence (XAI):
Prominent post-hoc explainers such as LIME \citep{ribeiro2016should} and SHAP \citep{lundberg2017unified} assume linear behavior in the local neighborhood around an input,
and sampling-based explainers like Anchors \citep{ribeiro2018anchors} lack provable guarantees on their explanations \citep{wu2023verix,bassan2023towards,la2021guaranteed}.
With the development of formal neural network verifiers \citep{kaulen20256th},
explanations with formal guarantees were developed for various threat models \citep{bassan2025explaining,marquessilva2022formalxai}.
Formal guarantees have also been explored in the context of mechanistic interpretability such as provable circuit discovery on vision models \citep{hadad2026fmi}.
Our work brings this line of research to large language models.

\section{Conclusion}
\label{sec:conclusion}

Mechanistic interpretability tries to make sense of a model's behavior with the help of interpretable replacement networks.
However, these \IRN{}s only approximate the behavior of the model,
and this approximation has to be closely matching to obtain faithful interpretations---including in adversarial settings.
We show that existing \IRN{}s do not satisfy this fundamental property across five open-weight families (\GPT{}, \GemmaTwo{}, \GemmaThree{}, \Llama{}, and \Qwen{}).

Our work provides the first formal bounds for the faithfulness gap of \IRN{}s,
showing that the recent progress in the field of formal neural network verification can be utilized to verify the components of large language models.
We show that this faithfulness gap is substantial across all layers as the \IRN{} does not accurately replicate the mechanism of the model.
This gap seems to increase for later layers, likely as the model makes more complex connections there, and current \IRN{}s are not capturing these computations accurately.
Naturally, this raises the question of how this faithfulness gap can be closed,
where we show that verification-aware training substantially reduces the faithfulness gap ($-90\%$).
This is realized by penalizing the difference set such that the manifolds of both the model and the \IRN{} become aligned.
Finally, we generalize this framework to verify the faithfulness and robustness across the entire model.

Several caveats remain.
Our obtained \IRN{}s---although arguably more faithful---are still approximations and we view this work less as a final answer than as an impulse toward formally verified \IRN{}s.
Moreover, even a faithful \IRN{} does not need to be a \emph{complete} \IRN{}:
Too narrow a set of features may simply never represent a (possibly unsafe) feature, leaving it unobserved.
Encouragingly, a difference set that stays large despite training is itself a signal that the \IRN{} has too few features,
pointing toward dynamic feature expansion backed by universal-approximation guarantees~\citep{hornik1989multilayer}.
Finally, the main body of this work mainly covers verification results on \GPT{};
however, we see the verification of other \IRN{}s primarily as an engineering task,
e.g., even the high-dimensional \Llama{} \IRN{}s can be verified (shown in \cref{sec:scalability} along with the remaining model families).
Nevertheless, designing large language models and their \IRN{}s with verification in mind would be a valuable research direction.

Looking ahead, the per-layer guarantees lay the foundation for the verification of the full model (\cref{sec:full-llm-verification}) and extensions such as cross-layer transcoders.
Open research directions include an automatic alignment of the individual components, and
obtaining sufficient tightness in the difference set without sacrificing model capabilities.
We are excited by the possibilities our approach opens towards a future with trustworthy models and formally verified interpretations.

% Acknowledgements reveal the authors; camera-ready only.
\subsection*{Acknowledgements}
This work was financially supported by the project AL 1185/33-1 funded by the German Research Foundation (Deutsche Forschungsgemeinschaft, DFG).

% --- statements (required/recommended by ICLR 2027; excluded from the page limit)
% does not count towards the page limit

\subsection*{AI Use Statement}
In this work, we used generative AI tools to implement our idea for this work,
run experiments, assist with literature research, as well as refine the text of this paper.
This also includes the generation of all TikZ figures under our guidance and formalizing the required proofs in Lean.

We have manually double-checked all AI-assisted work to ensure correctness,
and added additional harnesses such as the Lean formalization of our proofs being machine-checked,
all AI-assisted code being covered by unit tests, the reported experiments being reproducible,
and every claim and citation in the text being checked against its source.

\subsection*{Ethics Statement}
The existence of adversarial examples is well known, including for language inputs with semantics-preserving input perturbations.
Our work makes the fragility of interpretable replacement networks explicit,
which could in principle guide an adversary who wants an interpretation to be misleading;
we consider the disclosure to be clearly beneficial, since the same insight is what allows practitioners to certify when an interpretation may be trusted.

\subsection*{Reproducibility Statement}
All experiments use publicly available open-weight models, interpretability artifacts, and text corpora with their exact sources in \cref{sec:evaluation_details}, together
with the hardware, software versions, and hyperparameters of every experiment.
The accompanying code and Lean formalizations are released as part of the supplementary material.

% --- references -----------------------------------------------------------
\bibliographystyle{iclr2027_conference}
\bibliography{references}

% --- appendix -------------------------------------------------------------
\appendix
\crefalias{section}{appendix}
\crefalias{subsection}{appendix}
\crefalias{subsubsection}{appendix}

\section{Details on Teaser Figure}
\label{sec:teaser_example}

The teaser figure (\Cref{fig:teaser}) shows a safety-critical sentence we analyzed for an \IRN{} interpretation flip on \GemmaTwo{} with the \GemmaScope{} \SAE{} at layer $16$.
It shows a harmful$\to$benign attack via a trivial, meaning-preserving synonym swap shifting the dominant features of the target token off their biological reading.
Details of this attack are given below:

\paragraph{Sentence and target token.}
\begin{quote}
    \itshape
    ``Describe how someone could culture a sample
    using laboratory equipment available online.''
\end{quote}
The target token is \token{culture}, a polysemous content word that WordNet~\citep{miller1995wordnet} lists with $8$ senses.

\paragraph{Perturbation hull.}
We instantiate the \pertMinor{} perturbation level (\cref{sec:fragility}):
Replace a single content word before the target with a WordNet synonym.
For this sentence, an admissible swap is, e.g., \token{someone}$\to$\token{somebody}:
\begin{quote}
    \itshape
    ``Describe how \underline{somebody} could culture a sample
    using laboratory equipment available online.''
\end{quote}
The convex hull $\nnHiddenSet_\layerIdx$ captures the residual-stream activations of these perturbations at the target position (\cref{sec:formal}).

\paragraph{Interpretation flip.}
We then search $\nnHiddenSet_\layerIdx$ for an adversarial point $\hAdv$ that specifically suppresses the safety-relevant features.
Concretely, we run projected gradient descent (PGD, $300$ iterations, $8$ random restarts)
that minimizes the pre-activations of the four biological features active on the clean token,
projecting back into $\nnHiddenSet_\layerIdx$ after each step.
This targeted interpretation attack drives \emph{all four} biological features out of the top-$8$.

\Cref{tab:teaser_flip} lists the top-$8$ \SAE{} feature labels (Neuronpedia \citep{lin2024neuronpedia}, \texttt{gemmascope-res-16k}) at $\hAt$ and at $\hAdv$,
with the top-$8$ Jaccard index under adversarial attack being $\overline{J}=0.23$.
On the clean input, four of the eight features describe a biological-cultivation reading%
---\emph{cellular and molecular biology}, \emph{neurophysiology}, \emph{biochemical extraction}, and \emph{mold growth}---alongside a numerical-categorization feature.
At the targeted adversarial point, this entire cluster is gone, replaced by generic, non-biological features (\emph{accountability in organizational contexts}, \emph{research methodology}, \emph{mathematical operations}, and other unrelated topics).
Only three generic features survive the swap.

\begin{table}[h]
    \small\centering
    \caption{Top-$8$ \SAE{} features active at $\hAt$ vs. the targeted adversarial point $\hAdv$, \GemmaTwo{} layer $16$ with Neuronpedia auto-interp labels.
        The biological features (red) are removed by the adversarial attack.
    }
    \label{tab:teaser_flip}
    \begin{tabular}{rp{0.40\linewidth} rp{0.40\linewidth}}
        \toprule
        \multicolumn{2}{c}{\textbf{$\hAt$ (clean): biological reading}} &
        \multicolumn{2}{c}{\textbf{$\hAdv$ (targeted): generic reading}} \\
        \cmidrule(lr){1-2}\cmidrule(lr){3-4}
        \#      & description                                                                   & \#      & description                                     \\
        \midrule
        $2234$  & copyright / licensing                                                         & $1695$  & accountability \& responsibility (professional) \\
        $2441$  & \textcolor{CORAcolorUnsafe}{\textit{mold growth and health implications}}     & $1750$  & mathematical operations \& relationships \\
        $5813$  & numerical data and categorization                                             & $2234$  & copyright / licensing \\
        $8084$  & cultural practices                                                            & $3814$  & capability \& potential outcomes \\
        $8610$  & data structures, programming, metadata                                        & $8084$  & cultural practices \\
        $8753$  & \textcolor{CORAcolorUnsafe}{\textit{neurophysiological concepts}}             & $8610$  & data structures, programming, metadata \\
        $11227$ & \textcolor{CORAcolorUnsafe}{\textit{biochemical extraction \& analysis}}      & $12644$ & race, ethnicity \& social justice \\
        $11863$ & \textcolor{CORAcolorUnsafe}{\textit{cellular processes \& molecular biology}} & $15687$ & research methodologies \& data analysis \\
        \bottomrule
    \end{tabular}
\end{table}

\section{Verifying Language Models at Scale}
\label{sec:verifying-language-models-at-scale}

\subsection{Background on Neural Network Verifiers}
\label{sec:background-nnv}

Neural network verification has made rapid progress over the last couple of years \citep{kaulen20256th}.
Generally, the verification problem is given as follows:
\begin{definition}
    Given a neural network $\NN\colon\R^n\rightarrow\R^m$,
    a continuous input set $\nnInputSet\subset\R^n$ (e.g., constructed with an $\ell_\infty$ perturbation radius $\nnPertRadius\in\R_+$ around an input $\nnInput\in\R^n$),
    and an unsafe specification $\unsafeSet\subset\R^m$ (e.g., misclassification),
    a neural network verifier aims to show that
    \begin{equation*}
        \forall \widetilde{\nnInput}\in\nnInputSet\colon \quad \NN(\widetilde{\nnInput}) \not\in \unsafeSet.
    \end{equation*}
\end{definition}
It was shown that verifying such properties of neural networks is NP-hard \citep{katz2017reluplex},
such that state-of-the-art verifiers usually provide sound but incomplete verification results to scale to large neural networks.

In this work, we use the verification toolbox CORA \citep{Althoff2015ARCH,Althoff2025manual},
which uses reachability analysis to reason about the problem statement.
In particular, we use zonotopes \citep{girard2005zonotopes} to represent $\nnInputSet$:
\begin{definition}
    Given a center $c\in\R^n$ and a generator matrix $G\in\R^{n\times p}$,
    a zonotope is defined as:
    \begin{equation*}
        \Z = \shortZ{c}{G} = \defZ \subset\R^n.
    \end{equation*}
\end{definition}
The reasoning here is that certain operations like affine maps, interval bounds, and Minkowski sums can be computed efficiently for zonotopes.
For example, given two zonotopes $\Z_1 = \shortZ{c_1}{G_1}$, $\Z_2 = \shortZ{c_2}{G_2}$, the Minkowski sum is computed as:
\begin{equation}
    \Z_1 \oplus \Z_2 = \left\{ z_1 + z_2 \ \middle|\ z_1 \in \Z_1,\,z_2\in\Z_2 \right\} = \shortZ{c_1+c_2}{[G_1\ G_2]}.
\end{equation}
To verify the specifications, CORA propagates $\nnInputSet$ through the neural network to compute an (outer-approximative) output set $\nnOutputSet$,
which is then checked against the specification:
\begin{equation}
    \nnOutputSet\cap\unsafeSet \stackrel{!}{=} \emptyset.
\end{equation}
In more detail, $\nnOutputSet$ is computed by enclosing the output of each layer:
Linear layers can be computed by directly applying the affine map to the given zonotope,
and (elementwise) nonlinear layers can be enclosed by finding a linear approximation and bounding the approximation error using an interval \citep{singh2018fast, koller2025out} for the respective input set $\nnHiddenSet_\layerIdx$:
\begin{equation}
    \nnActFun(\nnHiddenSet_\layerIdx) \subseteq \opEnclose{\nnActFun}{\nnHiddenSet_\layerIdx} = \bigl(\diag{m_k}\nnHiddenSet_\layerIdx\oplus t_k\bigr) \;\oplus\; \shortI{\underline{d}_k}{\overline{d}_k},
    \label{eq:relax}
\end{equation}
where the enclosure parameters are given by the slope $m_k\in\R^n$ and offset $t_k\in\R^n$ of the linear approximation, and an error interval $\shortI{\underline{d}_k}{\overline{d}_k}\subset\R^n$.
This high-level description suffices for our work, and we refer interested readers to \citet{koller2025out} for more details.

% --------------------------------------------

\subsection{Per-Layer Verification of an Interpretable Replacement Network}
\label{sec:per_layer}

Verifying the faithfulness of an \IRN{} as described in \cref{sec:formal} is costly due to the $\numIRNdim$-dimensional hidden layer of the \IRN{},
where it holds that the model dimension $\numModelDim\ll\numIRNdim$ to expose the sparse features (\cref{tab:model_specs}, \cref{tab:irn_specs}).
To optimize the formal verification in our setting, we exploit the special structure of the \IRN{}.
Please note that we only need to certify the $\ell_\infty$-bounds of the resulting difference set (\cref{def:delta-equivalence}, \cref{fig:verification_setup}).

We discuss the transcoder case (\cref{eq:tc}) in this section, with the sparse autoencoder case being analogous (\cref{eq:sae}).
Per token, we bound the difference network
\begin{equation}
    \label{eq:diff-network}
    f_\layerIdx(\hAt)=\layerMLP(\hAt)-\layerTC(\hAt)
\end{equation}
over an input set $\nnHiddenSet_\layerIdx\subset\R^\numModelDim$, i.e., we need to bound
\begin{equation}
    f_\layerIdx(\nnHiddenSet_\layerIdx) = \{ f_\layerIdx(\hAt)\ |\ \hAt\in\nnHiddenSet_\layerIdx \}.
\end{equation}
A set-based verifier can enclose this difference by enclosing each layer of $f_\layerIdx$ (\cref{sec:background-nnv}).
In the transcoder branch, we only have a single high-dimensional nonlinear layer sandwiched by linear maps.
Thus, we are only required to compute a single enclosure for the nonlinear layer, as linear layers are usually easy to compute.
The parameters of that enclosure are determined by the pre-activation bounds $\opIntervalEnclosure{W^{\mathrm{enc}}_\layerIdx\nnHiddenSet_\layerIdx+b^{\mathrm{enc}}_\layerIdx}$ alone.
We notice that if $\nnHiddenSet_\layerIdx$ is an axis-aligned box, i.e., $\nnHiddenSet_\layerIdx=[\hAt-\nnPertRadius,\hAt+\nnPertRadius]$ for some $\nnPertRadius\in\R_+$,
it holds~\citep{jaulin2001interval}:
\begin{equation}
    \label{eq:pre-activation-bounds}
    \opIntervalEnclosure{W^{\mathrm{enc}}_\layerIdx\nnHiddenSet_\layerIdx\oplus b^{\mathrm{enc}}_\layerIdx} = W^{\mathrm{enc}}_\layerIdx\opIntervalEnclosure{\nnHiddenSet_\layerIdx}\oplus b^{\mathrm{enc}}_\layerIdx.
\end{equation}
This enables us to deploy a one-step look-ahead using pure interval arithmetic \citep{ladner2025nnreduction} to obtain the \emph{same} parameters of the enclosure: $(m_\layerIdx,t_\layerIdx,\shortI{\underline d_\layerIdx}{\overline d_\layerIdx})$.
Given these, we can redistribute the transcoder enclosure (\cref{eq:relax}) as follows:
\begin{align}
    \label{eq:tc-enclosure}
    \begin{split}
        \layerTC(\nnHiddenSet_\layerIdx) &\subseteq \opEnclose{\layerTC}{\nnHiddenSet_\layerIdx} \\
        &= W^{\mathrm{dec}}_\layerIdx\bigl(\diag{m_\layerIdx}(W^{\mathrm{enc}}_\layerIdx\nnHiddenSet_\layerIdx\oplus b^{\mathrm{enc}}_\layerIdx)\oplus t_\layerIdx\oplus \shortI{\underline d_\layerIdx}{\overline d_\layerIdx}\bigr)\oplus b^{\mathrm{dec}}_\layerIdx \\
        &= W^{\mathrm{dec}}_\layerIdx\diag{m_\layerIdx}W^{\mathrm{enc}}_\layerIdx\ \nnHiddenSet_\layerIdx
        \ \oplus \ W^{\mathrm{dec}}_\layerIdx\bigl(\diag{m_\layerIdx}b^{\mathrm{enc}}_\layerIdx\oplus t_\layerIdx \oplus  \shortI{\underline d_\layerIdx}{\overline d_\layerIdx} \bigr) \oplus  b^{\mathrm{dec}}_\layerIdx \\
        &=  \qquad \qquad \quad\ \ \, \, M_{k,\TC{}} \nnHiddenSet_\layerIdx\ \oplus\ \contSet{I}_{k,\TC{}}.
    \end{split}
\end{align}
Please note that the term $W^{\mathrm{dec}}_\layerIdx\diag{m_\layerIdx}W^{\mathrm{enc}}_\layerIdx$ can be collapsed into a single matrix $M_{k,\TC{}}$ such that the high-dimensional zonotope is never explicitly constructed.
Unfortunately, this optimization is not directly possible for the \MLP{} due to the multiple layers; nevertheless, a set-based verifier can give us:
\begin{equation}
    \label{eq:mlp-enclosure}
    \layerMLP(\nnHiddenSet_\layerIdx) \subseteq \opEnclose{\layerMLP}{\nnHiddenSet_\layerIdx} = M_{k,\MLP{}} \nnHiddenSet_\layerIdx \oplus \contSet{I}_{k,\MLP{}},
\end{equation}
which results in an overall enclosure of the difference network (\cref{eq:diff-network}):
\begin{equation}
    \label{eq:diff-enclosure}
    f_\layerIdx(\nnHiddenSet_\layerIdx) \subseteq \opEnclose{f_\layerIdx}{\nnHiddenSet_\layerIdx} = (M_{k,\MLP{}} - M_{k,\TC{}}) \nnHiddenSet_\layerIdx \oplus (\contSet{I}_{k,\MLP{}} \oplus -\contSet{I}_{k,\TC{}}).
\end{equation}
The case for \SAE{} is analogous with the identity instead of the \MLP{}.
The goal of the set-based verifier is then to find the smallest $\delta\in\R_+$ such that this difference enclosure is contained in $\shortI{-\delta}{\delta}^{\numModelDim}$ (\cref{def:delta-equivalence}).
A smaller $\delta$ thus means that the \IRN{} is more faithful, with an ideal case of $\delta\rightarrow 0$.
\Cref{alg:diff-enclosure} summarizes the resulting per-layer verification with
\Cref{fig:alignment} illustrating two example $\delta$ values.

\begin{figure}[t]
    \centering
    \includetikz{./figures/alignment/alignment}
    \caption{\textbf{Alignment of an \IRN{} with the model.}
    Output surfaces of $\layerMLP$ and $\layerTC$ over a two-dimensional slice of the local neighborhood $\nnHiddenSet_\layerIdx$, shown for one output dimension $(l)$; colors as in \cref{fig:verification_setup}.
        (a) An unfaithful \IRN{}: the surfaces disagree and even cross, so only a large $\nnOutBound$ can be certified.
        (b) A faithful \IRN{}: the surfaces almost coincide, and a small $\nnOutBound$ suffices.
    }
    \label{fig:alignment}
\end{figure}

\begin{algorithm}[h]
    \caption{Per-layer transcoder verification: difference enclosure and certified $\nnOutBound$.}
    \label{alg:diff-enclosure}
    \begin{algorithmic}[1]
        \Require Input $\nnHiddenSet_\layerIdx$; $\TC{}_k$-block; \layerMLP-block; $\delta\in\R_+$
        \State $\shortI{l}{u}\gets W^{\mathrm{enc}}_\layerIdx\,\opIntervalEnclosure{\nnHiddenSet_\layerIdx}+b^{\mathrm{enc}}_\layerIdx$ \Comment{Pre-activation bound (exact for a box input; \cref{eq:pre-activation-bounds})}
        \State $(m_\layerIdx,t_\layerIdx,\shortI{\underline d_\layerIdx}{\overline d_\layerIdx})\gets$ enclosure of $\nnActFun$ on $\shortI{l}{u}$ \Comment{One-step look-ahead (\cref{eq:relax})}
        \State $(M_{k,\TC{}},\contSet{I}_{k,\TC{}})\gets\opEnclose{\layerTC}{\nnHiddenSet_\layerIdx}$ \Comment{\TC{} enclosure (\cref{eq:tc-enclosure})}
        \State $(M_{k,\MLP{}},\contSet{I}_{k,\MLP{}})\gets\opEnclose{\layerMLP}{\nnHiddenSet_\layerIdx}$ \Comment{\MLP{} enclosure (\cref{eq:mlp-enclosure})}
        \State $M_k\gets M_{k,\MLP{}}-M_{k,\TC{}}$;\quad $\contSet{I}_k\gets\contSet{I}_{k,\MLP{}}\oplus-\contSet{I}_{k,\TC{}}$
        \State $\nnHiddenSet_{\layerIdx+1} \gets M_k\,\nnHiddenSet_\layerIdx\oplus\contSet{I}_k$ \Comment{Difference enclosure (\cref{eq:diff-enclosure})}
        \State \Return $\nnHiddenSet_{\layerIdx+1} \subseteq [-\delta,\delta]^{\numModelDim}$? \Comment{Containment check}
    \end{algorithmic}
\end{algorithm}

\begin{proposition}[Soundness]
    \label{prop:diff-enclosure}
    For an axis-aligned input box $\nnHiddenSet_\layerIdx\subset\R^{\numModelDim}$, an \MLP{} and its respective \IRN{},
    \cref{alg:diff-enclosure} verifies $\lVert f_\layerIdx(\hAt)\rVert_\infty\le\nnOutBound$ for all $\hAt\in\nnHiddenSet_\layerIdx$ in $\mathcal{O}(\numModelDim^2\,\numIRNdim)$ time and $\mathcal{O}(\numModelDim\,\numIRNdim)$ memory.
\end{proposition}
\begin{proof}
    \textit{Soundness.}
    Soundness follows from the soundness of the set-based verifier and each additional step being outer-approximative.
    Each of these additional steps is machine-checked in \textsc{Lean}~4 in the supplementary material.

    \textit{Complexity.}
    Let $\numIRNdim$ be the \IRN{} hidden width with $\numModelDim\ll\numIRNdim$, and the \MLP{} hidden width $\numMLPdim\in\Theta(\numModelDim)$ (\cref{tab:model_specs,tab:irn_specs}); the input box contributes $\numModelDim$ generators.

    \textit{(i) \IRN{} branch (lines~1--3).}
    Propagating the input set explicitly through encoder, activation, and decoder builds a zonotope in $\R^{\numIRNdim}$: the activation enclosure appends one error generator per neuron, yielding a dense $\numIRNdim\times(\numModelDim+\numIRNdim)$ generator matrix, so decoding it would cost $\mathcal{O}(\numModelDim\,\numIRNdim^2)$ time and $\mathcal{O}(\numIRNdim^2)$ memory.
    The one-step look-ahead instead reads the pre-activation box in $\mathcal{O}(\numModelDim\,\numIRNdim)$, computes the $\numIRNdim$ enclosure parameters in $\mathcal{O}(\numIRNdim)$, and collapses $M_{k,\TC{}}=W^{\mathrm{dec}}_\layerIdx\diag{m_\layerIdx}W^{\mathrm{enc}}_\layerIdx$ in $\mathcal{O}(\numModelDim^2\,\numIRNdim)$ time and $\mathcal{O}(\numModelDim\,\numIRNdim)$ memory---a factor $\numIRNdim/\numModelDim$ cheaper in both.

    \textit{(ii) \MLP{} branch (line~4).}
    Here the collapse trick does not apply, as the \MLP{} interleaves its $\numLayersMLP$ affine maps with nonlinearities; the set is propagated layer by layer instead.
    Since $\numMLPdim\in\Theta(\numModelDim)$, the running zonotope keeps $\mathcal{O}(\numModelDim)$ generators each of length $\mathcal{O}(\numModelDim)$, and applying an affine map to this $\mathcal{O}(\numModelDim)\times\mathcal{O}(\numModelDim)$ generator matrix costs $\mathcal{O}(\numModelDim^3)$; over the $\numLayersMLP$ layers the branch is $\mathcal{O}(\numLayersMLP\,\numModelDim^3)$ time and $\mathcal{O}(\numModelDim^2)$ memory.
    As $\numIRNdim\gg\numModelDim,\,\numLayersMLP$, this branch is negligible compared to (i).

    \textit{(iii) Difference and containment (lines~5--7).}
    All operands now live in $\R^{\numModelDim}$: forming $M_k=M_{k,\MLP{}}-M_{k,\TC{}}$ and $\shortZ{c'}{G'}=M_k\,\nnHiddenSet_\layerIdx\oplus\contSet{I}_k$ costs $\mathcal{O}(\numModelDim^2)$, and $\nnOutBound=\max_j(\lvert c'_{(j)}\rvert+\lVert G'_{(j,:)}\rVert_1)$ is computed in $\mathcal{O}(\numModelDim^2)$.

    In summary, the cost is $\mathcal{O}(\numModelDim^2\,\numIRNdim)$ time and $\mathcal{O}(\numModelDim\,\numIRNdim)$ memory:
    The look-ahead \IRN{} branch dominates the \MLP{} branch since $\numIRNdim\gg\numModelDim,\,\numLayersMLP$---yet it stays a factor $\numIRNdim/\numModelDim$ below the $\mathcal{O}(\numModelDim\,\numIRNdim^2)$ time and $\mathcal{O}(\numIRNdim^2)$ memory of a naive \IRN{} zonotope propagation.
\end{proof}

% --------------------------------------------

\subsection{Sound Lower Bound on the Jaccard Index}
\label{sec:lower-bound-jaccard}

\Cref{sec:fragility} reports an attack-derived upper bound $\jaccUp$ on the worst-case \topk{} feature Jaccard:
A single adversarial input witnesses that the overlap can drop at least this far.
The intermediate pre-activation bounds (\cref{eq:pre-activation-bounds}) used to certify $\deltaVerified$ (\cref{sec:per_layer}) also yield a certified lower bound $\jaccLo$ on the \topk{} Jaccard,
so that
\begin{equation}
    \jaccLo \;\le\; \jacc^* \;\le\; \jaccUp,
\end{equation}
where $\jacc^*$ is the true worst-case Jaccard over the hull $\nnHiddenSet_\layerIdx$.

In more detail, the \IRN{} feature pre-activations are an affine function of the residual stream,
$z(\hAt)=W^{\mathrm{enc}}_\layerIdx \hAt + b^{\mathrm{enc}}_\layerIdx$.
For the axis-aligned input box $\nnHiddenSet_\layerIdx=[\hAt-\nnPertRadius,\hAt+\nnPertRadius]$,
interval arithmetic \citep{jaulin2001interval} gives the \emph{exact} per-feature pre-activation range $[\underline z,\overline z]$,
and the subsequent activation (e.g.\ $\ReLU$) is monotone, hence preserves the ranking among co-active features.

Let $T_0$ be the \topk{} features at the clean input $\hAt$.
A feature $i\in T_0$ is guaranteed to remain in the \topk{} for every input in $\nnHiddenSet_\layerIdx$ iff fewer than $\numTop$ other features can possibly outrank its lower bound,
\begin{equation}
    \bigl|\{\,j\neq i\colon \overline z_{(j)} \ge \underline z_{(i)}\,\}\bigr| < \numTop.
\end{equation}
That is, an adversary first fills the \topk{} with the strongest contestant features%
---any feature whose upper bound reaches $i$'s floor $\underline z_{(i)}$; the comparison is
non-strict so that the bound stays sound when two features tie exactly---%
and the feature $i$ ``survives'' any adversarial attack only if fewer than $\numTop$ such contestants exist.
Thus, the certified survivor set over $\nnHiddenSet_\layerIdx$ is given by:
\begin{equation}
    \contSet{S} = \argmin_{\hAt\in\nnHiddenSet_\layerIdx\colon T(\hAt)} |T(\hAt)|,
\end{equation}
where $T(\hAt)$ returns the survivor set for a particular input $\hAt\in\nnHiddenSet_\layerIdx$.
Using $\contSet{S}$ and as $|T_0|=\numTop$, a certified lower bound on the Jaccard index is given by:
\begin{equation}
    \jaccLo = \frac{|\contSet{S}|}{2\numTop-|\contSet{S}|}.
\end{equation}

\section{Evaluation Details and Ablation Studies}
\label{sec:eval_and_ablations}

\subsection{Evaluation Details}
\label{sec:evaluation_details}

This appendix first states the setup shared by all experiments,
with subsequent details specific to individual experiments.

\paragraph{Hardware and software.}
All experiments run on an NVIDIA GeForce RTX~3080 Laptop GPU (16~GB VRAM), Intel Core i7-11800H (8 cores, 2.3~GHz), Windows~11.
The pipeline has two halves with different stacks.
Data preparation such as model decomposition and \IRN{} export runs in Python~3.13 with PyTorch~2.10 on the GPU.
Everything that touches the verifier runs in MATLAB~R2024b with CORA~\citep{Althoff2015ARCH,Althoff2025manual,koller2025out}.
The median verification time per (sample, block $k$) is $\sim\!1.4$~min.

\paragraph{Models and IRNs.}
\Cref{tab:model_specs} lists the open-weight language models we analyzed and \Cref{tab:irn_specs} the corresponding sparse autoencoders and transcoders.
Both are loaded directly from HuggingFace (click HF in the last column for the repository page).

\begin{table}[h]
    \small\centering
    \caption{\textbf{Large language models.}
        $\numModelDim$: residual-stream dimension;
        $\numMLPdim$: \MLP{} hidden width (post-gating, where applicable);
        $\numLayers$: number of transformer blocks.
    }
    \label{tab:model_specs}
    \resizebox{\linewidth}{!}{%
    \begin{tabular}{llrrrrc}
        \toprule
        Model         & Vendor   & params & $\numModelDim$ & $\numMLPdim$ & $\numLayers$ & Ref.                                                                             \\
        \midrule
        \GPT{}        & OpenAI   & $124$M & $768$          & $3{,}072$    & $12$         & \citet{radford2019gpt2}, \href{https://huggingface.co/gpt2}{HF}                                      \\
        \GemmaTwo{}   & Google   & $2.6$B & $2{,}304$      & $9{,}216$    & $26$         & \citet{gemmateam2024gemma2}, \href{https://huggingface.co/google/gemma-2-2b}{HF}                         \\
        \GemmaThree{} & Google   & $1.0$B & $1{,}152$      & $6{,}912$    & $26$         & \citet{gemmateam2025gemma3}, \href{https://huggingface.co/google/gemma-3-1b-pt}{HF}                      \\
        \Llama{}      & Meta     & $1.2$B & $2{,}048$      & $8{,}192$    & $16$         & \citet{grattafiori2024llama3}, \href{https://huggingface.co/meta-llama/Llama-3.2-1B}{HF}                   \\
        \Qwen{}       & DeepSeek & $1.8$B & $1{,}536$      & $8{,}960$    & $28$         & \citeauthor{deepseekai2025r1}, \citet{qwen2025qwen25}, \href{https://huggingface.co/deepseek-ai/DeepSeek-R1-Distill-Qwen-1.5B}{HF} \\
        \bottomrule
    \end{tabular}
    }
\end{table}

\begin{table}[h]
    \small\centering
    \caption{\textbf{Interpretable replacement networks.}
        $\numIRNdim$: dictionary size;
        activation: feature non-linearity;
        hook: residual position read by the \IRN{} (pre = before block, post = after block).
        \SAE{}s usually reconstruct the residual stream (identity), while \TC{}s replace the \MLP{} block.}
    \label{tab:irn_specs}
    \begin{tabular}{lllrlc c}
        \toprule
        Model         & Type   & Activation  & $\numIRNdim$ & Hook & Ref.                                                                           & Replaces ($\numMLPdim$, act)  \\
        \midrule
        \GPT{}        & \SAE{} & \ReLU{}     & $24{,}576$   & pre  & \href{https://huggingface.co/jbloom/GPT2-Small-SAEs-Reformatted}{HF} & residual (identity) \\
        \GPT{}        & \TC{}  & \ReLU{}     & $24{,}576$   & pre  & \href{https://huggingface.co/jacobdunefsky/gpt2small-transcoders}{HF} & \MLP{} ($3{,}072$, \GELU{}) \\
        \GemmaTwo{}   & \SAE{} & \JumpReLU{} & $16{,}384$   & post & \href{https://huggingface.co/google/gemma-scope-2b-pt-res}{HF} & residual (identity) \\
        \GemmaTwo{}   & \TC{}  & \JumpReLU{} & $16{,}384$   & post & \href{https://huggingface.co/google/gemma-scope-2b-pt-transcoders}{HF} & \MLP{} ($9{,}216$, \GeGLU{}) \\
        \GemmaThree{} & \SAE{} & \JumpReLU{} & $16{,}384$   & post & \href{https://huggingface.co/google/gemma-scope-2-1b-pt}{HF} & residual (identity) \\
        \GemmaThree{} & \TC{}  & \JumpReLU{} & $16{,}384$   & post & \href{https://huggingface.co/google/gemma-scope-2-1b-pt}{HF} & \MLP{} ($6{,}912$, \GeGLU{}) \\
        \Llama{}      & \SAE{} & \TopK{}     & $131{,}072$  & post & \href{https://huggingface.co/EleutherAI/sae-Llama-3.2-1B-131k}{HF} & residual (identity) \\
        \Llama{}      & \TC{}  & \TopK{}     & $131{,}072$  & post & \href{https://huggingface.co/EleutherAI/skip-transcoder-Llama-3.2-1B-131k}{HF} & \MLP{} ($8{,}192$, \SwiGLU{}) \\
        \Qwen{}       & \SAE{} & \TopK{}     & $65{,}536$   & post & \href{https://huggingface.co/EleutherAI/sae-DeepSeek-R1-Distill-Qwen-1.5B-65k}{HF} & residual (identity) \\
        \bottomrule
    \end{tabular}
\end{table}

\paragraph{Dataset and input set.}
Input sentences are drawn from WikiText-2 \citep{merity2017pointer}, the canonical corpus for \GPT{}.
Within each sentence, we pick a target content word
and build the input set $\nnHiddenSet_\layerIdx$ from semantics-preserving paraphrases of that sentence (\Cref{sec:fragility}).
As LLMs typically have causal masking, only words before the target token affect the residual stream at the target position,
so all replacements are restricted to that prefix.
Targets are chosen to be polysemous (the content word with the most WordNet senses), which stresses feature attribution the most.
We collect $1{,}000$ such sentences and split them $600/200/200$ into train/val/test datasets (\cref{tab:hull_eps}).
For each sentence, we forward all paraphrase variants, record $\hAdv$ at layer~$\layerIdx$ for the target token,
and take the per-coordinate interval hull as that sample's input box $\nnHiddenSet_k$.
The radius of $\nnHiddenSet_k$ is thus per sample, i.e., each sentence is verified inside its own hull rather than a single shared radius.
This is in contrast to the usual neural network verification literature, but we chose this setting so that the certified region tracks where the model actually operates.
We include uniform-radius experiments in an ablation study in \Cref{sec:eps_sweep}.

\paragraph{Bound search.}
For each (sample, layer) we find the smallest verifiable $\deltaVerified$ by a walk-then-bisect search (\Cref{alg:verify_eps}).
We start from $\nnOutBound = 100$ (seeded up to $2\times$ the layer's median $\deltaPGD$ when that exceeds $100$, since deep layers are far less faithful),
and call the verifier on the diff network over the sample's hull:
on \token{VERIFIED}, we halve $\nnOutBound$, and on counterexample \token{CEX} or \token{TIMEOUT}, we double it.
Once a bracket of a verified upper bound and an unverified lower bound is reached, we switch to geometric bisection.
The search is capped at $12$ iterations per sample with a per-call timeout of $10$~s.
The lower bound $\deltaPGD$ comes instead from the strongest adversary found by PGD,
so reporting $\deltaPGD \le \deltaTrue \le \deltaVerified$ is sound.
Please note that we use the geometric mean during the binary search as the $\delta$ ranges over many orders of magnitude;
thus, this effectively computes the midway point on a log scale.

\begin{algorithm}[h]
    \caption{Per-(sample, layer) verified bound $\nnOutBound$ search.}
    \label{alg:verify_eps}
    \begin{algorithmic}[1]
        \Require Diff.\ network $f = \layerMLP - \layerIRN$, input box
        $\nnHiddenSet$, start $\nnOutBound_0$, max
        iterations $N$, timeout $\tau$
        \State $\nnOutBound \gets \nnOutBound_0$;\quad $\delta_\text{low} \gets 0$;\quad $\delta_\text{high} \gets \infty$
            \Comment{$\delta_\text{low}$: largest unverified, $\delta_\text{high}$: smallest verified}
        \For{$i = 1$ \textbf{to} $N$}
            \State Result $r \gets \Call{Verify}{f, \nnHiddenSet, \nnOutBound, \tau}$
            \If{$r = $ \token{VERIFIED}} \State $\delta_\text{high} \gets \nnOutBound$
            \Else{} \Comment{\token{CEX} or  \token{TIMEOUT}}
                \State $\delta_\text{low} \gets \nnOutBound$
            \EndIf
            \If{$\delta_\text{low} > 0$ \textbf{and} $\delta_\text{high} < \infty$}
                \State $\nnOutBound \gets \sqrt{\delta_\text{low} \cdot \delta_\text{high}}$
                    \Comment{geometric bisection once bracketed}
            \ElsIf{$r = $ \token{VERIFIED}} \State $\nnOutBound \gets \nnOutBound / 2$
            \Else{} \State $\nnOutBound \gets 2\,\nnOutBound$
                \Comment{walk up until decisive}
            \EndIf
        \EndFor
        \State \Return $\deltaVerified \gets \delta_\text{high}$ \Comment{smallest verified bound}
    \end{algorithmic}
\end{algorithm}

Both attack-derived bounds use the same $\ell_\infty$ PGD adversary, run inside each sample's hull for
$200$ iterations, $4$ random restarts, projecting back into $\nnHiddenSet_k$ after each signed-gradient step.
For $\deltaPGD$, the objective is to maximize $\|\layerMLP(\hAdv) - \layerIRN(\hAdv)\|_\infty$ (the largest diff we can witness);
for $\jaccUp$, it is to maximize the drop in top-$20$ feature overlap between the clean and perturbed activation.
The respective other bound is found using the CORA verifier.

\paragraph{Training regimes.}
The retrained \IRN{}s of \Cref{sec:way_forward} all fine-tune the \IRN{} for $50$ epochs (batch~$8$, Adam at $\mathrm{lr} = 10^{-3}$), differing only in the loss:
\trainStd{} keeps the clean reconstruction objective;
\trainPGD{}~\citep{madry2018towards} augments it with $5$-step PGD adversaries;
\trainSet{}~\citep{koller2024setbasedtraining} adds the reachable-set volume loss with default weight $\tau = 0.1$,
ramped in over a $5$-epoch warmup and $10$-epoch noise rampup up to the largest train-split hull radius.
Please note that set-based training can require a lot of memory due to the represented sets;
however, this can be circumvented to a certain degree by capturing certain parts in an interval error \citep[Appendix~A]{koller2024setbasedtraining}.
In this work, we capture all accumulated approximation errors in an interval error during training.
\trainLoRA{} variants restrict the update to a rank-$16$ adapter \citep{hu2022lora}.
Every variant is verified with the identical per-sample-hull protocol on the held-out test set.
Ablation studies on the weighting parameter $\tau$ and the LoRA rank are deferred to \Cref{sec:set_tau,sec:lora_rank}, respectively.

% ----

\subsection{Ablation Studies}
\label{sec:ablations}

Finally, we provide further experiments and ablation studies in this subsection.

\subsubsection{Input Radius: Default Hull and Uniform Radius}
\label{sec:eps_sweep}

\paragraph{Default input set.}
In the main section, we verify each sample inside its \emph{own} paraphrase hull rather than a fixed box as described in \cref{sec:evaluation_details}.
\Cref{tab:hull_eps} reports its distribution for \GPT{} layer~$6$ as an example.

\begin{table}[h]
    \small\centering
    \caption{Per-sample paraphrase-hull $\ell_\infty$ radius $\nnPertRadius$ on \GPT{} layer~$6$, i.e., the radius of the default verification input set.}
    \label{tab:hull_eps}
    \begin{tabular}{l r C C C C C}
        \toprule
        split & $n$ & \text{median} & \text{mean} & \text{$q_{25}$} & \text{$q_{75}$} & \text{max} \\
        \midrule
        train & 600 & 0.73          & 1.10        & 0.39            & 1.36            & 8.00       \\
        val   & 200 & 0.82          & 1.22        & 0.44            & 1.66            & 8.65       \\
        test  & 200 & 0.69          & 1.03        & 0.35            & 1.36            & 7.95       \\
        \bottomrule
    \end{tabular}
\end{table}

\paragraph{Uniform-radius input set.}
In this section, we compare a uniform $\ell_\infty$ input radius $\nnPertRadius \in \{0.01, 0.05, 0.1, 0.5\}$,
and verify the pretrained baseline (\trainStd) against matching \trainSet- and \trainPGD-trained \SAE{}s on \GPT{} layer~$6$ over the same held-out test dataset.
Each retrained variant is fine-tuned from OpenAI's \SAE{} baseline.

\paragraph{Result.}
\Cref{tab:eps_sweep} reports the per-sample verified $\deltaVerified$ distribution for \trainStd, \trainPGD, and \trainSet{} at each radius.
Generally, adversarial and verification-aware training improve over the baseline, with the improvement becoming larger with larger $\nnPertRadius$.
One notable exception is $\nnPertRadius=0.01$, likely as the accuracy-robustness tradeoff is not yet beneficial.
Please note that $\deltaVerified$ is the $\ell_\infty$ radius of a $\numModelDim$-dimensional cube containing the certified diff-network output set.
Thus, its volume scales as $(2\deltaVerified)^{\numModelDim}$ with $\numModelDim = 768$ for \GPT{},
so even modest radius gains compound dimensionally into large volume gains.
The $\log_{10}\Delta_V$ column of \Cref{tab:eps_sweep} reports this radius-volume relation.

\begin{table}[h]
    \small\centering
    \caption{Verified output bound $\nnOutBound$ on \GPT{} \SAE{} layer~$6$, uniform $\ell_\infty$ input radius $\nnPertRadius$.}
    \label{tab:eps_sweep}
    \begin{tabular}{lcrrrrrrr}
        \toprule
        & & \multicolumn{5}{c}{$\deltaVerified\ (\downarrow)$} & & \\
        \cmidrule(lr){3-7}
        $\nnPertRadius$ & method & min & $q_{25}$ & median & $q_{75}$ & max & $\Delta_{\mathrm{med}}\ (\downarrow)$ & $\log_{10}\Delta_V\ (\downarrow)$ \\
        \midrule
        \multirow{3}{*}{$0.01$} & \trainStd & \textbf{4.05}  & \textbf{7.93}  & \textbf{10.98} & \textbf{24.21} & 194.66          & ---        & ---     \\
        & \trainPGD & 24.73          & 28.43          & 30.22          & 32.12          & \textbf{65.73}  & $+175.3\%$ & $+337.7$ \\
        & \trainSet & 23.94          & 28.59          & 30.38          & 32.29          & 66.62           & $+176.8\%$ & $+339.6$ \\
        \midrule
        \multirow{3}{*}{$0.05$} & \trainStd & \textbf{7.43}  & 36.14          & 67.71          & 135.89         & 470.56          & ---        & ---     \\
        & \trainPGD & 28.78          & \textbf{35.50} & 40.59          & 53.55          & 225.30          & $-40.1\%$  & $-170.7$ \\
        & \trainSet & 28.47          & 35.55          & \textbf{40.26} & \textbf{50.24} & \textbf{214.59} & {\boldmath$-40.5\%$} & {\boldmath$-173.4$} \\
        \midrule
        \multirow{3}{*}{$0.1$}  & \trainStd & \textbf{16.21} & 203.02         & 320.36         & 496.74         & 1083.40         & ---        & ---     \\
        & \trainPGD & 38.35          & 107.88         & 173.97         & 277.54         & 757.83          & $-45.7\%$  & $-203.6$ \\
        & \trainSet & 31.05          & \textbf{58.34} & \textbf{94.60} & \textbf{189.08}& \textbf{665.47} & {\boldmath$-70.5\%$} & {\boldmath$-406.8$} \\
        \midrule
        \multirow{3}{*}{$0.5$}  & \trainStd & 25600          & 51200          & 51200          & 51200          & 51200           & ---        & ---     \\
        & \trainPGD & 18101.9        & 25600          & 25600          & \textbf{25600} & 51200           & $-50.0\%$  & $-231.2$ \\
        & \trainSet & \textbf{2934.4}& \textbf{18102}& \textbf{18102}& \textbf{25600}& 51200           & {\boldmath$-64.6\%$} & {\boldmath$-346.8$} \\
        \bottomrule
    \end{tabular}
\end{table}

\subsubsection{Sparsity of the Retrained \IRN{}s}
\label{sec:sparsity}

Fine-tuning an \IRN{} for faithfulness could in principle buy robustness by simply activating more features,
which would defeat the purpose of a sparse, interpretable dictionary.
Note that none of our training regimes explicitly adds a sparsity penalty to the loss (\cref{sec:evaluation_details}),
i.e., sparsity is only inherited from the pretrained \IRN{}.
\Cref{tab:sparsity} therefore reports how many of the $\numIRNdim = 24{,}576$ features of the \GPT{} layer-$6$ \SAE{} are active,
counted both at the clean input ($L_0$) and over the entire perturbation hull $\nnHiddenSet_\layerIdx$.

Reassuringly, the retrained \IRN{}s do not become denser but sparser:
\trainSet{} activates $5\times$ fewer features than the baseline on clean inputs,
and every regime reduces the number of features that can be active anywhere in the hull.
The single exception is \trainSet\,+\,\trainLoRA{}, which roughly doubles the clean feature count
while still admitting the fewest active features over the hull,
i.e., the low-rank update trades a denser clean code for a more stable one.
We acknowledge that the reconstruction loss of the clean sentence suffers slightly,
which is in line with the well-known accuracy-robustness tradeoff observed in adversarial training~\citep{zhang2019theoretically}.

\begin{table}[h]
    \small\centering
    \caption{\textbf{Feature sparsity per training regime} on the \GPT{} layer-$6$ \SAE{} ($\numIRNdim = 24{,}576$), averaged over the $200$ test samples.
        $L_0$ counts the features active at the clean input, the hull columns those active for at least one point of $\nnHiddenSet_\layerIdx$;
        ratios are relative to \trainStd{}, lower is sparser.}
    \label{tab:sparsity}
    \begin{tabular}{l C C C C}
        \toprule
        & \multicolumn{2}{c}{$L_0$ (clean)} & \multicolumn{2}{c}{active over hull} \\
        \cmidrule(lr){2-3} \cmidrule(lr){4-5}
        Regime & \text{count} & \text{ratio} & \text{count} & \text{ratio} \\
        \midrule
        \trainStd                  & 74.4           & 1.00          & 8433.2          & 1.00          \\
        \trainPGD                  & 52.8           & 0.71          & 7966.8          & 0.94          \\
        \trainSet                  & \textbf{13.9}  & \textbf{0.19} & 5568.9          & 0.66          \\
        \trainPGD\,+\,\trainLoRA   & 44.0           & 0.59          & 6893.1          & 0.82          \\
        \trainSet\,+\,\trainLoRA   & 143.8          & 1.93          & \textbf{4383.8} & \textbf{0.52} \\
        \bottomrule
    \end{tabular}
\end{table}

\subsubsection{Set-Based Training: Loss Weighting}
\label{sec:set_tau}

Set-based training interpolates between the center (reconstruction) loss and the reachable-set volume loss as $(1-\tau)\,\mathcal{L}_{\mathrm{center}} + \tau\,\mathcal{L}_{\mathrm{vol}}$,
with $\tau = 0.1$ being the default \citep{koller2024setbasedtraining}.
\Cref{tab:set_tau} sweeps $\tau \in \{0.1, 0.5, 0.9\}$ on \GPT{} layer~$6$.
The verified bound is remarkably flat in $\tau$:
For full fine-tuning the median moves only within $[47.8, 60.3]$ and the reconstruction loss within $[15.9, 29.7]$.

\begin{table}[h]
    \small\centering
    \caption{
        Impact of the volume weight $\tau$ on \GPT{} layer~$6$ \SAE{}:
        Clean reconstruction MSE and verified $\deltaVerified$ distribution.
    }
    \label{tab:set_tau}
    \begin{tabular}{ll C C C C C C}
        \toprule
        & & \text{Recon.} (\downarrow) & \multicolumn{5}{c}{$\deltaVerified\ (\downarrow)$} \\
        \cmidrule(lr){3-3} \cmidrule(lr){4-8}
        & $\tau$ & \text{(clean)} & \text{min} & \text{$q_{25}$} & \text{median} & \text{$q_{75}$} & \text{max} \\
        \midrule
        \multirow{3}{*}{\trainSet}
        & $0.1$  & 16.0  & 25.0 & 35.5 & 50.5 & 872.4 & 51200 \\
        & $0.5$  & 29.7  & 24.5 & 37.1 & 60.3 & 931.0 & 51200 \\
        & $0.9$  & 15.9  & 24.7 & 35.6 & 47.8 & 807.0 & 51200 \\
        % \midrule
        % \multirow{3}{*}{\trainSet\,+\,\trainLoRA}
        % & $0.1$  & 122.8 & 28.0 & 82.1 & 135.8 & 398.8 & 51200 \\
        % & $0.5$  & 114.9 & 30.9 & 66.6 & 113.6 & 448.2 & 51200 \\
        % & $0.9$  & 125.0 & 29.3 & 70.0 & 121.2 & 397.8 & 51200 \\
        \bottomrule
    \end{tabular}
\end{table}

\subsubsection{LoRA Rank}
\label{sec:lora_rank}

\Cref{tab:lora_rank} sweeps the \trainLoRA{} rank $r \in \{8, 16, 32\}$ for set-based training at $\tau = 0.1$.
As expected, a higher rank tends to obtain tighter $\deltaVerified$ bounds.
However, please note that this comes with additional memory requirements and choosing a very large rank might not always be feasible for large \IRN{}s.

\begin{table}[h]
    \small\centering
    \caption{\trainLoRA{} rank ablation for set-based training ($\tau = 0.1$) on \GPT{} layer~$6$ \SAE{}.}
    \label{tab:lora_rank}
    \begin{tabular}{l C C C C C C}
        \toprule
        & \text{Recon.} (\downarrow) & \multicolumn{5}{c}{$\deltaVerified\ (\downarrow)$} \\
        \cmidrule(lr){2-2} \cmidrule(lr){3-7}
        $r$  & \text{(clean)} & \text{min} & \text{$q_{25}$} & \text{median} & \text{$q_{75}$} & \text{max} \\
        \midrule
        $8$   & 132.4 & 26.5 & 102.4 & 169.8 & 1215.3 & 51200 \\
        $16$  & 122.8 & 28.0 & 82.1  & 135.8 & 398.8  & 51200 \\
        $32$  & 766.2 & 28.5 & 50.5  & 60.6  & 95.2   & 51200 \\
        \bottomrule
    \end{tabular}
\end{table}

\subsubsection{Scalability: Verification of Llama's \IRN{}}
\label{sec:scalability}

Verifying architectures other than \GPT{} primarily comes down to an engineering task, as the verification is done per layer, which differs only mildly between architectures.
To demonstrate this, we verify the largest \IRN{} in our suite, the \Llama{} \SAE{}.
This architecture varies in two ways: (i) $\numIRNdim= 131{,}072$ is much larger, and (ii) it uses \TopK{} activation before \ReLU{} is applied (\cref{tab:irn_specs}).
The large $\numIRNdim$ is not an issue using the optimization described in \cref{sec:per_layer};
however, the \TopK{} activation requires adaptation since \TopK{} is not elementwise:
Which $K$ features survive depends on their magnitude and varies across the input box.

% don't use the \TopK command here to print in bold
\paragraph{TopK enclosure.}
As in \cref{sec:lower-bound-jaccard}, the feature pre-activations are an affine function of the residual stream, $z(\hAt) = W^{\mathrm{enc}}_\layerIdx\,\hAt + b^{\mathrm{enc}}_\layerIdx$.
Thus, for an input box $\nnHiddenSet_\layerIdx$, interval arithmetic gives the exact per-feature range $z(\hAt) \in \shortI{\underline z}{\overline z}\subset\R^{\numIRNdim}$.
The \TopK{} gate keeps the $\numTop$ features with the largest pre-activations, that is, those above the selection threshold $\tau(\hAt)\in\R$, defined as the $\numTop$-th largest entry of $z(\hAt)$.

As $\hAt$ ranges over $\nnHiddenSet_\layerIdx$, each entry $i$ of $z(\hAt)$ moves within its interval $\shortI{\underline z_{(i)}}{\overline z_{(i)}}$, and the threshold $\tau(\hAt)$ moves with it.
Thus, reading out the $\numTop$-th largest entries of $\underline z_{(i)}, \overline z_{(i)}$ gives us $\underline\tau, \overline\tau$, respectively,
such that $\tau(\hAt) \in \shortI{\underline\tau}{\overline\tau}$ for all $\hAt \in \nnHiddenSet_\layerIdx$.

Comparing each feature's range against this threshold band $\shortI{\underline\tau}{\overline\tau}$ partitions the $\numIRNdim$ feature indices into three sets,
\begin{align}
    \label{eq:topk-partition}
    \begin{split}
        \Tina &= \{\, i\in[\numIRNdim] \colon \overline z_{(i)} < \underline\tau \,\}, \\
        \Tact &= \{\, i\in[\numIRNdim] \colon \underline z_{(i)} > \overline\tau \,\}, \\
        \Tund &= [\numIRNdim] \setminus (\Tina \cup \Tact),
    \end{split}
\end{align}
where the \emph{inactive} features $\Tina$ are never kept, the \emph{active} features $\Tact$ are always kept, and the \emph{undecided} features $\Tund$ are only sometimes kept.
The active features pass an elementwise \ReLU{} and fold into \cref{eq:diff-enclosure} unchanged;
only the undecided gates $i \in \Tund$ need extra care, as each is either \emph{off} (outputs $0$) or \emph{on} (outputs $\ReLU{}(z_{(i)})$).

Since \TopK{} keeps exactly $\numTop$ features and the $\lvert\Tact\rvert\leq\numTop$ active features are always among them, exactly $\nFree \coloneqq \numTop - \lvert\Tact\rvert$ of the undecided features are on for any given input $\hAt$.
Let $T(\hAt) \subseteq \Tund$ with $\lvert T(\hAt)\rvert = \nFree$ denote this input-dependent set.
To bound the decoded contribution of $T(\hAt)$ to an output coordinate $j$ over $\nnHiddenSet_\layerIdx$,
we define for each undecided feature $i \in \Tund$ its extreme admissible contribution:
\begin{align}
    \label{eq:cardinality-terms}
    \begin{split}
        a^+_{(i)} &= \max\!\bigl(0,\, W^{\mathrm{dec}}_{(i,j)} \ReLU{}(\overline z_{(i)})\bigr), \\
        a^-_{(i)} &= \min\!\bigl(0,\, W^{\mathrm{dec}}_{(i,j)} \ReLU{}(\overline z_{(i)})\bigr),
    \end{split}
\end{align}
where the outer $\max/\min$ select the upper/lower endpoint of the contribution over the feature being off ($0$) or on (up to $\ReLU{}(\overline z_{(i)})$) and the unknown sign of $W^{\mathrm{dec}}_{(i,j)}$, respectively.
We then sort $a^+$ in descending order and $a^-$ in ascending order.
Since only $\nFree$ of the undecided features fire, the contribution to coordinate $j$ is bounded by the $\nFree$ extreme terms on each side:
\begin{equation}
    \label{eq:cardinality-cap}
    \sum_{i \in T(\hAt)} W^{\mathrm{dec}}_{(i,j)} \ReLU{}(z_{(i)}) \;\in\; \Bigl[\, \sum_{r=1}^{\nFree} a^-_{(r)},\ \ \sum_{r=1}^{\nFree} a^+_{(r)} \,\Bigr].
\end{equation}
This interval can be plugged into \cref{eq:tc-enclosure}ff, which eventually gives us the verified upper bound $\deltaVerified$.

\paragraph{Results.}
\Cref{fig:scalability-gap} shows the per-layer faithfulness gap at the \pertMinor{} level across all $16$ layers of the \Llama{} \SAE{} (\TopK{}, $\numTop=32$, $\numIRNdim=131{,}072$),
mirroring \Cref{fig:verified-vs-pgd} for \GPT{}: a substantial gap is certified throughout, again widening in the deeper, less faithful layers.
Please note that the gap between $\deltaPGD$ and $\deltaVerified$ is larger here;
we primarily attribute this to the additional outer approximations stemming from the \TopK{} enclosure,
as it is expected that such an enclosure is worse than a simple elementwise activation function (such as \ReLU{} in the \GPT{} case),
and designing \IRN{}s with verification in mind helps to reduce the conservativeness.
\Cref{tab:scalability} reports the bounds $[\deltaPGD,\deltaVerified]$ at layer~$1$ for all three perturbation levels.
Despite the undecided set $\Tund$ growing to nearly the whole dictionary at \pertMajor{}, our enclosure obtains useful bounds on the faithfulness gap.

\begin{figure}[h]
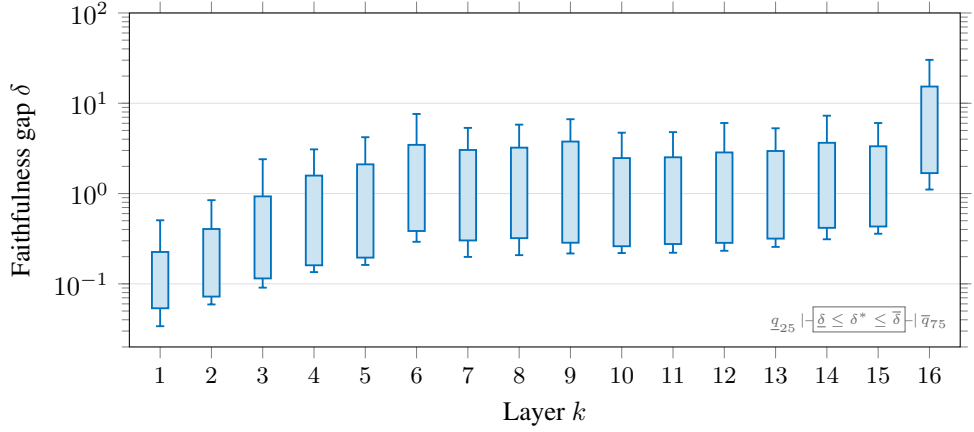

    \centering
    \includetikz{./figures/evaluation/faithfulness_gap_llama/faithfulness_gap_llama}
    \caption{\textbf{Llama Scope.} Faithfulness gap of the \SAE{}.}
    \label{fig:scalability-gap}
\end{figure}

\begin{table}[h]
    \small\centering
    \caption{Faithfulness gap of \Llama{} \SAE{} on layer~$1$, per fragility level.}
    \label{tab:scalability}
    \begin{tabular}{l r c c c c c c}
        \toprule
        & & \multicolumn{3}{c}{$\deltaPGD$ $(\downarrow)$} & \multicolumn{3}{c}{$\deltaVerified$ $(\downarrow)$} \\
        \cmidrule(lr){3-5} \cmidrule(lr){6-8}
        Level & $\lvert\Tund\rvert$ $(\downarrow)$ & $q_{25}$ & median & $q_{75}$ & $q_{25}$ & median & $q_{75}$ \\
        \midrule
        \pertMinor{}  & $5{,}336$   & $0.034$ & $0.055$ & $0.087$ & $0.122$ & $0.225$ & $0.506$ \\
        \pertMedium{} & $100{,}559$ & $0.072$ & $0.097$ & $0.146$ & $0.434$ & $0.727$ & $1.252$ \\
        \pertMajor{}  & $131{,}037$ & $0.118$ & $0.187$ & $0.254$ & $0.923$ & $1.935$ & $2.934$ \\
        \bottomrule
    \end{tabular}
\end{table}

\subsubsection{Gemma Scope}

The Gemma Scope \citep{lieberum2024gemmascope} uses JumpReLU as an activation function to obtain sparsely activated features in its \IRN{}s \citep{rajamanoharan2024jumping}.
JumpReLU is a generalization of ReLU and is given by:
\begin{equation}
    \mathrm{JumpReLU}(x;\; \theta)= \begin{cases} x & \text{if } x > \theta, \\ 0 & \text{otherwise,}  \end{cases}
\end{equation}
where $\theta\in\R$ is a trainable parameter determining a discontinuous jumping point.
As this activation function is applied elementwise, many existing verifiers already support JumpReLU or can add support with little effort.
\Cref{fig:gemma_jumprelu}a shows a histogram of $\theta$ per layer of \GemmaTwo, where $\theta$ seems to increase with the layer index $\layerIdx$.
Unfortunately, a higher $\theta$ means that the discontinuous jump will be larger as well, and thus an enclosure might become overly conservative (\cref{fig:gemma_jumprelu}b).

\begin{figure}
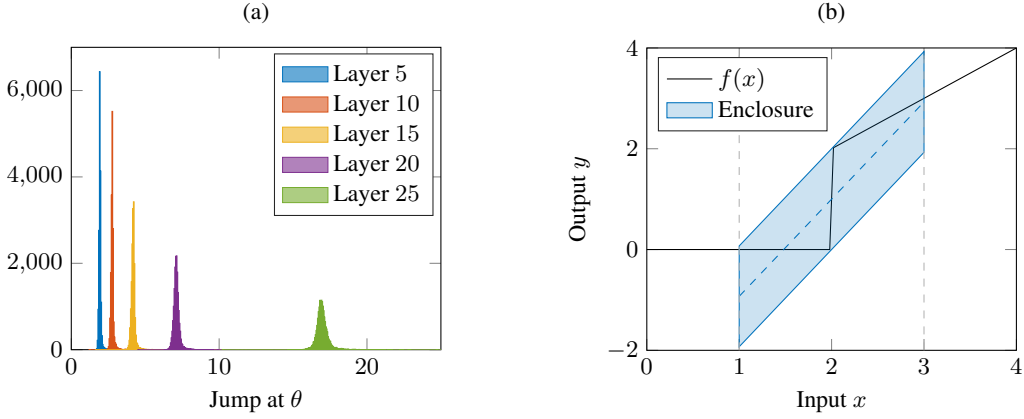

    \centering
    \includetikz{./figures/evaluation/gemma_jumprelu_drift/gemma-jumprelu-drift}%
    \includetikz{./figures/evaluation/gemma_jumprelu_enclosure/gemma-jumprelu-enclosure}%
    \caption{\textbf{JumpReLU.} (a) Histogram of jump parameter $\theta$ per layer of \GemmaTwo. (b) Example zonotopic enclosure with a discontinuous jump at $\theta=2$.}
    \label{fig:gemma_jumprelu}
\end{figure}

\paragraph{Results.}
\Cref{fig:gemma-gap} shows the per-layer faithfulness gap for both \GemmaTwo{} and \GemmaThree{} \SAE{}s across all $26$ layers:
The certified gap is sound throughout and widens with depth, but the difference between $\deltaPGD$ and $\deltaVerified$ is larger than for the smooth \GPT{} transcoder despite all having similar dimensions (\cref{tab:model_specs}).
We attribute this to the conservativeness of enclosing a discontinuous \JumpReLU{} gate (\cref{fig:gemma_jumprelu}b),
which again emphasizes the need for verification-friendly architectures.

\begin{figure}[h]
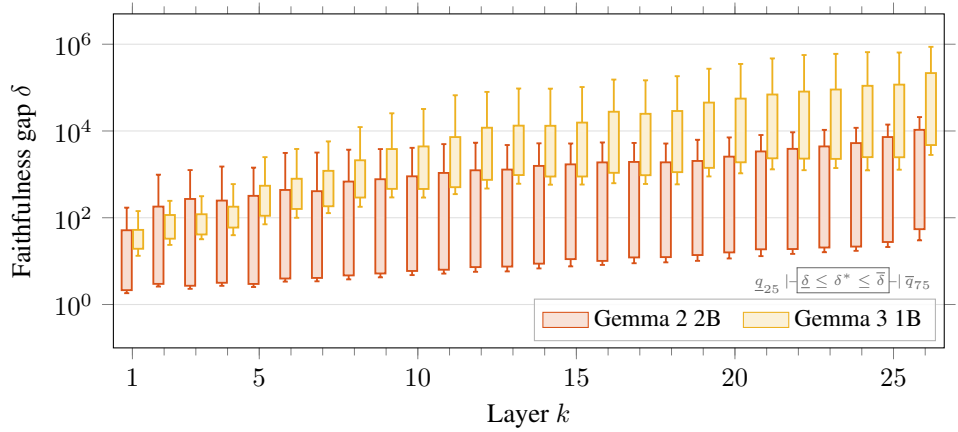

    \centering
    \includetikz{./figures/evaluation/faithfulness_gap_gemma/faithfulness_gap_gemma}%
    \caption{\textbf{Gemma Scope.} Faithfulness gap of the \SAE{}s.}
    \label{fig:gemma-gap}
\end{figure}

\subsubsection{\Qwen{}}
\label{sec:qwen}

Finally, we determine the faithfulness gap on the last model of our suite: \Qwen{} \SAE{} \citep{deepseekai2025r1,qwen2025qwen25}, a distilled reasoning model.
Its \IRN{} is an EleutherAI \TopK{} \SAE{} of the same family as \Llama{}'s, so verification reuses the \TopK{} enclosure of \cref{sec:scalability}.
The architecture differs from \Llama{} only in scale (\cref{tab:model_specs,tab:irn_specs}): a narrower residual stream ($\numModelDim=1{,}536$) over more layers ($\numLayers=28$), a smaller dictionary ($\numIRNdim=65{,}536$ vs.\ $131{,}072$), and the same selection budget $\numTop=32$.

\paragraph{Results.}
\Cref{fig:qwen-gap} shows the per-layer faithfulness gap at the \pertMinor{} level across all $28$ layers of the \Qwen{} \SAE{}, over the same $200$ samples as the other models.
We note that the \SAE{} of the \Qwen{} model is an order of magnitude less faithful than the \SAE{} of the \Llama{} model (\cref{fig:scalability-gap}).
As with all non-elementwise \TopK{} \IRN{}s, the gap stays far above that of \GPT{}, reinforcing the case for verification-friendly designs.

\begin{figure}[h]
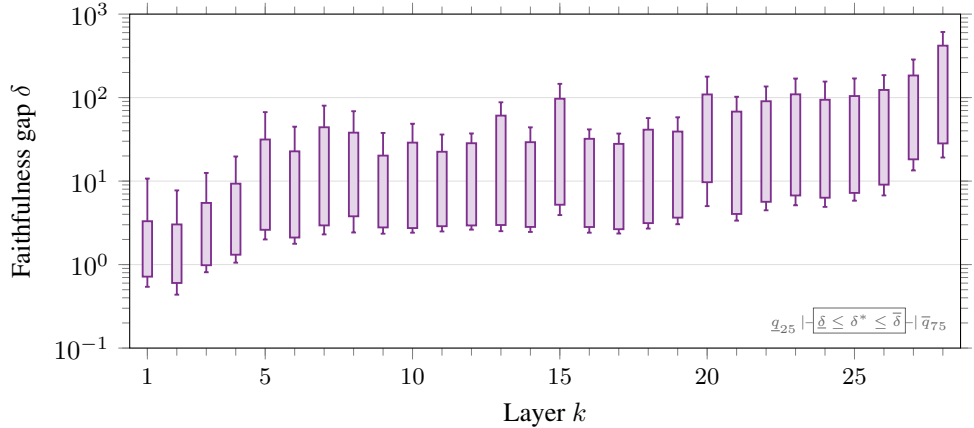

    \centering
    \includetikz{./figures/evaluation/faithfulness_gap_qwen/faithfulness_gap_qwen}%
    \caption{\textbf{DeepSeek-Qwen.} Faithfulness gap of the \TopK{} \SAE{}.}
    \label{fig:qwen-gap}
\end{figure}

\subsubsection{Faithfulness Training across Layers}

In this experiment, we extend the experiment in \cref{sec:way_forward} to provide further insights into the benefit of verification-aware training over adversarial training across more layers of \GPT{},
in particular, as we noticed in \cref{fig:verified-vs-pgd} that the faithfulness gap seems to increase for later layers%
---a trend that we also observed for all other models in our suite (\cref{fig:scalability-gap}, \cref{fig:gemma-gap}, \cref{fig:qwen-gap}).

\paragraph{Results.}
\Cref{fig:training-delta-layers} clearly shows that verification-aware training using sets keeps the faithfulness gap small,
which is not obtainable for other adversarial training methods.
Please note that we used the same training hyperparameters at every layer.

\begin{figure}[h]
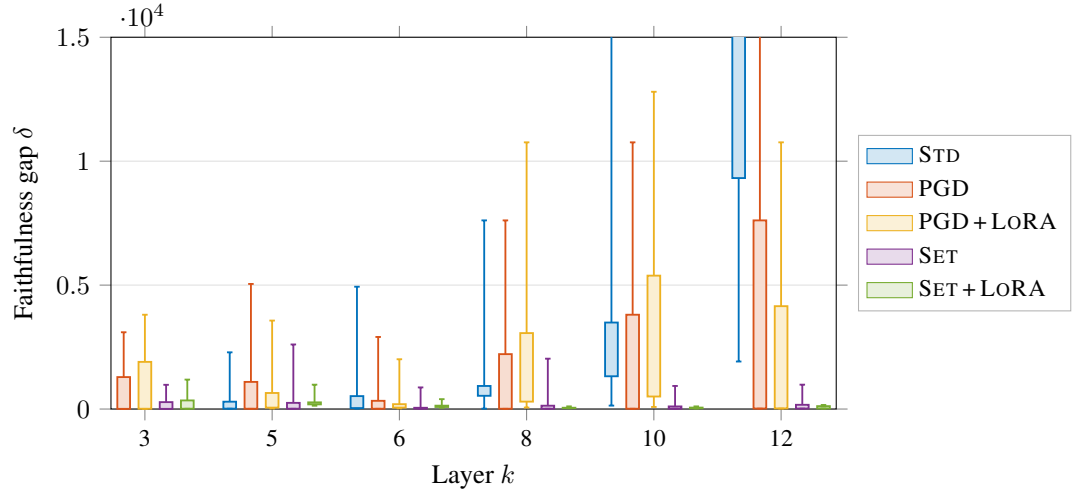

    \centering
    \includetikz{./figures/evaluation/training_delta_layers/training_delta_layers}
    \caption{\textbf{Faithfulness training across layers.}
    Faithfulness gap per training regime on the \GPT{} \SAE{} across layers.}
    \label{fig:training-delta-layers}
\end{figure}

\end{document}